\documentclass[twoside,11pt]{article}
\usepackage{jmlr2e}

\usepackage{microtype}
\usepackage{graphicx}
\usepackage{booktabs,bookmark} % for professional tables

\usepackage{hyperref}

\usepackage{algorithm}
\usepackage{algorithmic}

\usepackage{amsfonts}       % blackboard math symbols
\usepackage{nicefrac}       % compact symbols for 1/2, etc.

\usepackage{epstopdf}
\usepackage{enumitem}
\usepackage{courier}
\usepackage{mathtools}
\usepackage{color,amsmath,amssymb,multirow}
\usepackage{multicol}

\newcommand{\inner}[1]{\left\langle#1\right\rangle}

\def\A{\mathcal{A}}

\def\R{\mathbb{R}}

\def\D{\mathcal{D}}

\def\P{\mathcal{P}}

\newcommand{\norm}[1]{\left\|#1\right\|}

\def\grad{\mathop{\rm grad}\nolimits}

\def\bzero{{\mathbf 0}}

\def\bu{\mathbf u}

\def\bA{\mathbf A}
\def\bB{\mathbf B}
\def\bI{{\mathbf I}}

\def\bQ{\mathbf Q}

\def\bU{{\mathbf U}}
\def\bV{{\mathbf V}}
\def\bW{{\mathbf W}}
\def\bX{{\mathbf X}}
\def\bY{{\mathbf Y}}
\def\bZ{{\mathbf Z}}

\def\maxop{\mathop{\rm max}\limits} %max operator
\def\minop{\mathop{\rm min}\limits}

\newcommand{\range}{{\rm range}}

\def\min{\mathop{\rm min}\nolimits}

\newcommand{\trace}{{\rm trace}}

\newcommand{\OG}[1]{{\mathcal{O}({#1})}}

\newcommand{\hess}{\mathrm{Hess}}

\newcommand{\LabelFunc}[2]{\parbox[t]{#2\textwidth}%
	{\hspace*{\fill}{\vspace*{-0.3cm}#1}\hspace*{\fill}}}

\newcommand{\TwoLabels}[2]{\LabelFunc{#1}{0.48}\hfill\LabelFunc{#2}{0.48}}

\newcommand{\ThreeLabels}[3]{\LabelFunc{#1}{0.31}\hfill
                             \LabelFunc{#2}{0.31}\hfill\LabelFunc{#3}{0.31}}

\newcommand{\FourLabels}[4]{\LabelFunc{#1}{0.23}\hfill\LabelFunc{#2}{0.23}\hfill
			      \LabelFunc{#3}{0.23}\hfill\LabelFunc{#4}{0.23}}

\newcommand{\FiveLabels}[5]{\LabelFunc{#1}{0.17}\hfill\LabelFunc{#2}{0.17}\hfill\LabelFunc{#3}{0.17}\hfill\LabelFunc{#4}{0.17}\hfill\LabelFunc{#5}{0.17}}

\newif\iflongversion
 \longversionfalse

\ShortHeadings{}{}
\begin{document} 
\title{Structured low-rank matrix learning:\\algorithms and applications}
\author{\name Pratik Jawanpuria \email pratik.jawanpuria@microsoft.com\\
\name Bamdev Mishra \email bamdevm@microsoft.com}
\editor{\ }
\maketitle

\begin{abstract} 
%We propose a novel optimization approach for learning a low-rank matrix which is also constrained to be in a given linear subspace. Low-rank constraints are regularly employed in applications such as recommender systems and multi-task learning. In addition, several system identification problems require learning matrix with both low-rank and linear subspace constraints.  
%We transform a variational characterization of the classical nuclear norm regularized formulation to its dual formulation and model it as saddle point minimax problem. 
%Motivated by large-scale optimization setting, we solve it via  a rank constrained non-convex surrogate. This translates to an optimization problem on the Riemannian spectrahedron manifold. We exploit the Riemannian structure to propose efficient first-order and second-order algorithms. The duality theory allows us to compute the duality gap for a candidate solution and our approach easily accommodates popular non-smooth loss functions,  e.g., the absolute-value loss. We effortlessly scale on the Netflix data set on both matrix completion and robust matrix completion problems, obtaining state-of-the-art generalization performance. Additionally, we demonstrate the efficacy of our approach in Hankel matrix learning and multi-task learning problems. 

We consider the problem of learning a low-rank matrix, constrained to lie in a linear subspace, and introduce a novel factorization for modeling such matrices. A salient feature of the proposed factorization scheme is it decouples the low-rank and the structural constraints onto separate factors. We formulate the optimization problem on the  Riemannian spectrahedron manifold, where the Riemannian framework allows to develop computationally efficient conjugate gradient and trust-region algorithms. Experiments on problems such as standard/robust/non-negative matrix completion, Hankel matrix learning and multi-task learning demonstrate the efficacy of our approach. A shorter version of this work  has been published in ICML'18 \citep{mjaw18a}. 

%We propose a novel optimization framework for learning a low-rank matrix which is also constrained to lie in a linear subspace. Exploiting the duality theory, we present a factorization that decouples the low-rank and structural constraints onto separate factors. 
%The optimization problem is formulated on the Riemannian spectrahedron manifold, where the Riemannian framework allows to develop computationally efficient conjugate gradient and trust-region algorithms. 
%Our approach easily accommodates popular non-smooth loss functions, e.g., $\ell_1$-loss, and our algorithms are scalable to large-scale problem instances. 
%Experiments on problems such as Hankel matrix learning, non-negative matrix completion, and robust matrix completion demonstrate the efficacy of our approach. 

%The numerical comparisons show that our  algorithms outperform state-of-the-art in standard, robust, and non-negative matrix completion, Hankel matrix learning, and multi-task feature learning problems on various benchmarks.

%on the Netflix data set on both matrix completion and robust matrix completion problems, obtaining state-of-the-art generalization performance. Additionally, we demonstrate the efficacy of our approach in Hankel matrix learning and multi-task learning problems. 
\end{abstract} 

\section{Introduction}
Our focus in this paper is on learning structured low-rank matrices and we consider the following problem: 
\begin{equation} \label{eqn:genericPrimal0}
\begin{array}{ll}
	\minop_{\bW\in\R^{d\times T}} &  \displaystyle CL(\bY,\bW) + \norm{\bW}_{*}^2, \\
	 {\rm subject\ to}& \bW \in \D,
\end{array}
\end{equation}
where $\bY \in \R^{d\times T}$ is a given matrix, $L: \R^{d\times T}\times\R^{d\times T}  \rightarrow \R$ is a convex loss function, $\|\cdot\|_{*}$ denotes the nuclear norm regularizer, $C>0 $ is the {cost parameter}, and $\D$ is the {\it linear} subspace corresponding to structural constraints. It is well known the nuclear norm regularization promotes low rank solutions since $\norm{\bW}_{*}$ is equal to the $\ell_1$-norm on the singular values of $\bW$ \citep{Fazel01}. The linear subspace $\D$ in problem (\ref{eqn:genericPrimal0}) is represented as $\D \coloneqq\{\bW:\A(\bW)\diamondsuit\bzero\}$, where $\A:\R^{d\times T}\rightarrow \R^n$ is a linear map and $\diamondsuit$ represents equality ($=$) or greater than equal to ($\geq$) constraint.  

Low-rank matrices are commonly learned in several machine learning applications such as matrix completion \citep{abernethy09a,boumal11a}, multi-task learning \citep{Argyriou08,Zhang10,mjaw12a}, multivariate regression \citep{yuan07a, journee10a}, to name a few. In addition to the low-rank constraint, other structural constraints may exist, e.g., entry-wise {\it non-negative}/{\it bounded} constraints~\citep{kannan12a,marecek17a,fang17a}. Several linear dynamical system models require learning a low-rank {\it Hankel} matrix \citep{Fazel13, markovsky13a}. A Hankel matrix has the structural constraint that all its anti-diagonal entries are the same. In robust matrix completion and robust PCA problems \citep{Wright09}, the matrix is learned as a superimposition of a low-rank matrix and a sparse matrix. This sparse structure is modeled effectively by choosing the loss function as the $\ell_1$-loss~\citep{Cambier16}. 
Similarly, low-rank 1-bit matrix completion solvers employ the {logistic} loss function~\citep{Davenport14,Bhaskar15} to complete a $0/1$ matrix.

We propose a {generic} framework to the structured low-rank matrix learning problem (\ref{eqn:genericPrimal0}) that is well suited for handling a variety of loss functions $L$ (including non-smooth ones such as the $\ell_1$-loss), structural constraints $\bW\in\D$, and is scalable for large-scale problem instances. Using the duality theory, we introduce a novel modeling of structured low-rank matrix $\bW$ of rank $r$ as $\bW=\bU\bU^\top (\bZ+\bA)$, where $\bU\in\R^{d\times r}$ and $\bZ,\bA\in\R^{d\times T}$. Our factorization naturally decouples the low-rank and structural constraints on $\bW$. The low-rank of $\bW$ is enforced with $\bU$, the structural constraint is modeled by $\bA$, and the loss function specific structure is modeled by $\bZ$. The separation of low-rank and structural constraints onto separate factors makes the optimization conceptually simpler. 
To the best of our knowledge, such a decoupling of constraints has not been studied in the existing structured low-rank matrix learning literature~\citep{Fazel13,markovsky13a,yu14b,Cambier16,fang17a}. 

Our approach leads to an optimization problem on the {\it Riemannian spectrahedron} manifold. We exploit the Riemannian framework to develop computationally efficient conjugate gradient (first-order) and trust-region (second-order) algorithms. 
The proposed algorithms outperform state-of-the-art in robust and non-negative matrix completion problems as well as low-rank Hankel matrix learning application. Our algorithms readily scale to the Netflix data set, even with the non-smooth $\ell_1$-loss and $\epsilon$-SVR ($\epsilon$-insensitive support vector regression) loss functions.  
%The proposed algorithms outperform state-of-the-art in various matrix completion problems, low-rank Hankel matrix learning problems as well as multi-task feature learning applications. Our algorithms readily scale to the Netflix data set, even for the non-smooth $\ell_1$-loss and $\epsilon$-SVR ($\epsilon$-insensitive support vector regression) loss functions.  

The main contributions of the paper are: 
\begin{itemize}
\item[--] we propose a novel factorization $\bW=\bU\bU^\top (\bZ+\bA)$ for modeling structured low-rank matrices.
\item[--] we present a unified framework to learn structured low-rank matrix for several applications with different constraints and loss functions. 
\item[--] we develop Riemannian conjugate gradient and trust-region algorithms for our framework. Our algorithms obtain state-of-the-art generalization performance across applications.
\end{itemize}

%Our approach for (\ref{eqn:genericPrimal0}) exploits a well-studied {\it variational} characterization of the {\it trace} norm $\| \bW\|_*$ (i.e., the sum of the singular values of $\bW$) to propose a {\it novel} rank-constrained minimax problem formulation for (\ref{eqn:genericPrimal0}). The trace norm $\| \bW\|_*$ regularizer promotes low-rank solutions. 

The outline of the paper is as follows. We introduce our structured low-rank matrix learning framework in Section~\ref{sec:formulation}. 
Section~\ref{sec:optimizationAlgorithm} presents our optimization approach. Section~\ref{sec:specialized_formulations} discusses specialized formulations (for various applications) within from our framework. The empirical results are presented in Section~\ref{sec:empiricalResults}. A shorter version of this work has been published in ICML'18 \citep{mjaw18a}. Our codes are available at \url{https://pratikjawanpuria.com/} and \url{https://bamdevmishra.com/}. 
%All the proofs and additional experiments are provided in the supplementary material. 
We begin by discussing the related works in the next section.

\section{Related work}\label{sec:related}
\textbf{Matrix completion:} 
Existing low-rank matrix learning literature has been primarily focused on problem  (\ref{eqn:genericPrimal0}) with the square loss and in the absence of the structural constraint $\bW \in \D$. 
Singular value thresholding~\citep{cai10a}, proximal gradient descent~\citep{toh10a}, active subspace selection~\citep{Hsieh14} are some of the algorithms proposed to solve (\ref{eqn:genericPrimal0}) without the structural constraint $\bW \in \D$. 
Alternatively, several works \citep{wen12a,mishra14c,boumal15a} propose to learn a low-rank matrix by fixing the rank explicitly, i.e. $\bW=\bU\bV^\top$, where $\bU\in\R^{d\times r}$, $\bV\in\R^{T\times r}$ and the rank $r$ is fixed {\it a priori}. 

\textbf{Robust matrix completion:} A matrix completion problem where few of the observed entries are perturbed/outliers \citep{Cambier16}. %yan13a
\citet{candes11a} model the robust matrix completion problem as a convex program with separate low-rank and sparse constraints. \citet{he12a} propose to use $\ell_1$-loss as the loss function and impose only the low-rank constraint. They propose an online algorithm that learns a low-dimensional subspace. \citet{Cambier16} build on this and employ the pseudo-Huber loss as a proxy for the non-smooth $\ell_1$-loss. They solve the following optimization problem:
{%\small
\begin{equation}\label{eqn:rmcOrg}
\small\minop_{\substack{\bW\in\R^{d\times T},\\{\rm {rank}}(\bW)=r}}  \sum\limits_{(i,j)\in\Omega} C\sqrt{\delta^2 + (\bY_{ij}-\bW_{ij})^2} + \sum\limits_{(i,j)\in{\Omega^c}} \bW_{ij}^2,
\end{equation}
}where $\Omega$ is the set of observed entries,  $\Omega^c$ is the complement of set $\Omega$, $r$ is the given rank, and $\delta>0$ is a given parameter.  
For  small value of $\delta$, the pseudo-Huber loss is close to the $\ell_1$-loss. They develop a large-scale Riemannian conjugate gradient algorithm in the fixed-rank setting.

\textbf{Non-negative matrix completion:} Certain recommender system and image completion based applications desire matrix completion with non-negative entries \citep{kannan12a,sun13a,tsagkatakis16a,fang17a}, i.e $\bW\geq\bzero$. \citet{kannan12a} present a block coordinate descent algorithm that learns $\bW$ as $\bW=\bU\bV^\top$ for a given rank $r$.  Recently, \citet{fang17a} propose a large-scale alternating direction method of multipliers (ADMM) algorithm for (\ref{eqn:genericPrimal0}) with the square loss in this setting. In each iteration $i$, for a known matrix $\bB_i$, they need to solve the following key sub-problem
$$\minop_{\bW\in\R^{d\times T}} C\norm{\bW-\bB_i}^2_{F} + \frac{1}{2}\norm{\bW}_{*}^2.$$
They approximately solve it in closed form for a given rank $r$ via soft-thresholding of the singular values of $\bB_i$. 

\textbf{Hankel matrix learning:} The Hankel constraint is a linear equality constraint (denoted by $\mathcal{H}(\bW)=\bzero$). \citet{Fazel13} propose the ADMM approaches to solve  (\ref{eqn:genericPrimal0}) with the above constraint. On the other hand, \citet{yu14b} learn a low-rank Hankel matrix by relaxing the structural constraint with a penalty term in the objective function. They  solve the following optimization problem: 
$$\minop_{\bW\in\R^{d\times T}} C\norm{\bW-\bY}^2_{F} + \lambda\norm{\mathcal{H}(\bW)}_{F}^2 +  \frac{1}{2}\norm{\bW}_{*}.$$ 
Hence, they approximately enforce the structural constraint. They discuss a generalized gradient algorithm to solve the above problem.  
\citet{markovsky13a} model the low-rank Hankel matrix learning problem as a non-linear least square problem in the fixed rank setting and propose a second-order algorithm. % {\color{red} Their algorithm is }. 
%Instead of employing $R(\bW)$ in (\ref{eqn:genericPrimal0}), \citet{markovsky13a} learn a Hankel matrix by fixing the rank {\it a priori} and strictly enforcing the structural constraints.

\textbf{Multi-task feature learning:} The goal in multi-task feature learning \citep{Argyriou08} is to jointly learn a low-dimensional latent feature representation common across several classification/regression problems (tasks). The optimization approaches explored include alternate minimization \citep{Argyriou08}, accelerated gradient descent \citep{ji09} and mirror descent  \citep{mjaw11a}. 
Problems such as multi-class classification \citep{Amit07}, multi-variate regression \citep{yuan07a}, matrix completion with side information \citep{xu13a}, among others, may be viewed as special cases of multi-task feature learning.

In the following section, we present a unified framework for solving the above problems. 

\section{Structured low-rank matrix learning}\label{sec:formulation}

%We present our formulations for the problem of learning a structured low-rank matrix $\bW$ close to a given matrix $\bY$. 
%Let $w_t$ denote the $t^{th}$ column of $\bW$. 
% and $w_{ti}$ denote the $i^{th}$ row of column $w_t$. 
%\noteBM{The loss term $L(\bW,\bY)$ in (\ref{eqn:genericPrimal0}) is defined as $L(\bW,\bY)\coloneqq \sum_{t=1}^Tl(y_{t},w_{t})$, where $l:\R^d\times \R^d\rightarrow\R$ is a loss function convex in the second argument.}
%The loss term $L(\bW,\bY)$ in (\ref{eqn:genericPrimal0}) is defined as $L(\bW,\bY)\coloneqq \sum_{t=1}^T\sum_{i=1}^{d}l(y_{ti},w_{ti})$, where $l:\R\times \R\rightarrow\R$ is a loss function convex in the second argument. 
%The linear %\footnote{A more generic form of $\D$ is permissible within our framework: $\D \coloneqq \{\bW:\A_1(\bW)=b_1,\A_2(\bW)\leq b_2\}$). Its discussion is avoided in order keep the notations simple.} 
%subspace $\D$ in problem (\ref{eqn:genericPrimal0}) is represented as $\D \coloneqq\{\bW:\A(\bW)=\bzero\}$, where $\A:\R^{d\times T}\rightarrow \R^n$ is a linear map.  

For notational simplicity, we consider only equality structural constraint $\A(\bW)=\bzero$ to present the main results.  However, our framework also admits inequality constraints $\A(\bW)\ge0$. % Non-negative matrix completion, which has $\bW\geq\bzero$ constraint,  is one of the applications discussed in this work. 
We use the notation $\P^d$ to denote the set of $d\times d$ positive semi-definite matrices with \emph{unit} trace. 

Problem (\ref{eqn:genericPrimal0}) is a convex problem with linear constraint. In addition to the trace-norm regularizer, the loss function $L$ may also be non-smooth (as in the case of $\ell_1$-loss or $\epsilon$-SVR loss). Dealing with (\ref{eqn:genericPrimal0}) directly or characterizing the nature of its optimal solution in the general setting is non trivial. To this end, we propose an equivalent partial dual of (\ref{eqn:genericPrimal0}) in the following section. The use of dual framework often leads to a better understanding of the primal problem \citep{mjaw15a}. In our case, the duality theory helps in discovering a novel factorization of its optimal solution, which is not evident directly from the primal problem (\ref{eqn:genericPrimal0}). This subsequently helps in the development of computationally efficient algorithms. 

%Instead of optimizing problem (\ref{eqn:squaredTraceNorm}) directly, we propose to analyze its partial dual formulation. The duality theory helps discovering a novel factorization for learning structured low-rank matrices. It also lends further insights into the nature of optimal solution of (\ref{eqn:squaredTraceNorm}), which further helps in the development of efficient algorithms. 

% In the following, we derive and analyze a dual formulation of (\ref{eqn:variational}) that is suitable in large-scale structured matrix learning. 

\subsection{Decoupling of constraints and duality}\label{sec:minimaxformulation}
The following theorem presents a dual problem equivalent to the primal problem (\ref{eqn:genericPrimal0}). % and an expression of the optimal solution $\bar{\bW}$ of (\ref{eqn:genericPrimal0}). 

\begin{theorem}\label{thm:dual1}
Let $L^{*}$ be the Fenchel conjugate function of the loss: $L:\R^{d\times T}\rightarrow\R$, $v\mapsto L(\bY,v)$. An equivalent partial dual problem of (\ref{eqn:genericPrimal0}) with $\A(\bW)=\bzero$ constraint is
%\begin{equation} \label{eqn:dual11}
%\minop_{\Theta\in \P^d}\   g(\Theta),
%\end{equation}
%where $g :\P^d\rightarrow\R: \Theta \mapsto g(\Theta)$ is the convex function
%\begin{align}
%g(\Theta) \coloneqq  \maxop_{s\in\R^n}\ \sum_{t=1}^T \Big(\maxop_{z_t\in\R^{d}} - Cl_{t}^{*}\Big(\frac{-z_{t}}{C}\Big)\label{eqn:dual12}\\ 
%\qquad- \frac{1}{2}(z_t + a_t)^\top\Theta(z_t+a_t)\Big),\nonumber
%\end{align}
\begin{equation}\label{eqn:dual11}%\label{eqn:dual12}
\minop_{\Theta\in \P^d}\ \ \maxop_{\bZ\in\R^{d\times T},\,s\in\R^n}\ \  f(\Theta,\bZ,s), 
\end{equation}
where the function $f$ is defined as
\begin{equation}\label{eqn:dual12}
f(\Theta,\bZ,s) \coloneqq  - CL^{*}(\frac{-\bZ}{C}) - \frac{1}{2}\inner{\Theta(\bZ+\A^{*}(s)),\bZ+\A^{*}(s)},
\end{equation}
and $\A^{*}:\R^n\rightarrow\R^{d\times T}$ is the adjoint of $\A$. 
\end{theorem}
\begin{proof}
We first note the the following variational characterization of the squared trace norm regularizer from Theorem 4.1 in~\citep{Argyriou06}: 
\begin{align}
\|\bW\|_{*}^2 = \minop_{\Theta\in \P^d, \range(\bW)\subseteq \range(\Theta)}\ \inner{\Theta^{\dagger}\bW,\bW},\label{eqn:variationalThm}
\end{align}
where $\range(\Theta)=\{\Theta z : z\in\R^d\}$. For a given $\bW$ matrix, the optimal $\bar{\Theta}=\nicefrac{\sqrt{\bW\bW^\top}}{\trace({\sqrt{\bW\bW^\top}})}$.

By employing the result in (\ref{eqn:variationalThm}), problem (\ref{eqn:genericPrimal0}) can be shown to be equivalent to the following problem: 
\begin{align}
&\minop_{\Theta\in \P^d}\ \minop_{\bW\in\R^{d\times T}}\ CL(\bY,\bW)+\frac{1}{2}\inner{\Theta^{\dagger}\bW,\bW}+  i_{\range(\Theta)}(\bW)\label{eqn:squaredTraceNorm}\\ 
&{\rm \ subject\ to:\ }\A(\bW)=\bzero,\nonumber
\end{align}
where $i_H$ is the indicator function for set $H$. 

In the following, we derive the dual problem of the following sub-problem of (\ref{eqn:squaredTraceNorm}): 
\begin{align}
&\minop_{\bW\in\R^{d\times T}}\ CL(\bY,\bW)+\frac{1}{2}\inner{\Theta^{\dagger}\bW,\bW}+  i_{\range(\Theta)}(\bW)\label{eqn:subproblem}\\
& {\rm \ subject\ to:\ }\A(\bW)=\bzero,\nonumber. 
\end{align}
We now introduce auxiliary variable $\bU$ such that $\bU=\bW$. The dual variables with respect to the constraints $\bU=\bW$ and $\A(\bW)=\bzero$ are $\bZ\in\R^{d\times T}$ and $s\in\R^n$ , respectively. Hence, the Lagrangian ($Q$) is as follows: 
\begin{align}
Q(\bW,\bU,\bZ,s) =  CL(\bY,\bU)+\frac{1}{2}\inner{\Theta^{\dagger}\bW,\bW}+  i_{\range(\Theta)}(\bW) + \inner{\bZ,\bU-\bW}  -\inner{s,\A(\bW)}. 
\end{align}
The dual function $q(\bZ,s;\Theta)$ of (\ref{eqn:subproblem}) is defined as 
\begin{align}\label{appendix:eqn:dualFunction}
q(\bZ,s;\Theta)\coloneqq \minop_{\bW\in\R^{d\times T},\bU\in\R^{d\times T}} Q(\bW,\bU,\bZ,s)
\end{align}
Using the definition of the conjugate function~\citep{Boyd04}, we get 
\begin{align}\label{appendix:eqn:minAuxillary}
\minop_{\bU\in\R^{d\times T}}  CL(\bY,\bU) + \inner{\bZ,\bU} = -C L^{*}(\frac{-\bZ}{C}),
\end{align}
where $L^{*}$ be the Fenchel conjugate function of the loss: $L:\R^{d\times T}\rightarrow\R^{d\times T}$, $v\mapsto L(\bY,v)$. 

We next compute the minimizer of $Q$ with respect to $\bW$. From the definition of the adjoint operator, it follows that 
\begin{align*}
\inner{s,\A(\bW)} = \inner{\A^*(s),\bW}
\end{align*}
The minimizer of $Q$ with respect to $\bW$ satisfy the following conditions 
\begin{align}
& \frac{\partial}{\partial \bW}\Big(-\inner{\A^*(s),\bW} + \frac{1}{2}\inner{\Theta^{\dagger}\bW,\bW}  - \inner{\bZ,\bW}\Big) = \bzero, \ \textup{and}\\
& \bW \in \range(\Theta)
\end{align}
which implies, 
\begin{align}
\Theta^{\dagger}\bW = \bZ + \A^{*}(s),\ \textup{subject to }\bW \in \range(\Theta)
\end{align}
Thus, the expression of the minimizer of $Q$ with respect to $\bW$ is 
\begin{equation}\label{eqn:optSol}
\bW=\Theta(\bZ + \A^{*}(s)).
\end{equation}
Plugging the results (\ref{eqn:optSol}) and (\ref{appendix:eqn:minAuxillary}) in the dual function (\ref{appendix:eqn:dualFunction}), we obtain 
\begin{align*}
q(\bZ,s;\Theta) =  - CL^{*}(\frac{-\bZ}{C}) - \frac{1}{2}\inner{\Theta(\bZ+\A^{*}(s)),\bZ+\A^{*}(s)}.
\end{align*}
Thus, the following partial dual problem is equivalent to the primal problem (\ref{eqn:genericPrimal0})
$$\minop_{\Theta\in \P^d}\ \maxop_{\bZ\in\R^{d\times T},s\in\R^{n}} q(\bZ,s;\Theta).$$
\end{proof}

Our next result gives the expression of an optimal solution  of the primal  problem (\ref{eqn:genericPrimal0}). 
\begin{lemma}\label{thm:optimalSol}
(Representer theorem) Let $\{\bar{\Theta},\bar{\bZ},\bar{s},\}$ be an optimal solution of (\ref{eqn:dual11}). Then, an optimal solution $\bar{\bW}$ of  (\ref{eqn:genericPrimal0}) with $\A(\bW)=\bzero$ constraint is: 
$$\bar{\bW}=\bar{\Theta}(\bar{\bZ}+\A^{*}(\bar{s})).$$
\end{lemma}
\begin{proof}
The above result is obtained from the proof of Theorem~\ref{thm:dual1}. In particular, see (\ref{eqn:optSol}). 
\end{proof}
%\textbf{Proof:} The lemma follows from the  Lagrangian duality. The proof is in the supplementary material.$\square$\newline\newline
%\begin{proof}
%The lemma follows from the  Lagrangian duality. The proof is in the supplementary material.
%\end{proof}

\textbf{Remark 1:} $\Theta$ in (\ref{eqn:dual11}) is a positive semi-definite with unit trace. Hence, an optimal $\bar{\Theta}$ in (\ref{eqn:dual11}) is a low-rank matrix. 

\textbf{Remark 2:} Theorem \ref{thm:dual1} gives us an expression for an optimal solution $\bar{\bW}$ of (\ref{eqn:genericPrimal0}): it is a product of $\bar{\Theta}$ and $\bar{\bZ}+\A^{*}(\bar{s})$. 
%The low-rank constraint is enforced through $\bar{\Theta}$ and the structural constraint is enforced through $\bar{\bZ}+\A^{*}(\bar{s})$. Overall, this facilitates using simpler optimization techniques as compared to the case where both the constraints are enforced on a single variable. 
The low-rank constraint is enforced through $\bar{\Theta}$, the loss-specific structure (encoded in $L^{*}$) is enforced through $\bar{\bZ}$, and the structural constraint is enforced through $\A^{*}(\bar{s})$. Overall, such a decoupling of constraints onto separate variables facilitates the use of simpler optimization techniques as compared to the case where all the constraints are enforced on a single variable. 

As discussed earlier, an optimal  $\bar{\Theta}$ of (\ref{eqn:dual11}) is a low-rank positive semi-definite matrix. However, an algorithm for (\ref{eqn:dual11}) need not produce intermediate iterates that are low rank. For large-scale optimization, this observation as well as other computational efficiency concerns motivate a fixed-rank parameterization of $\Theta$ as discussed in the following section.

\subsection{A new fixed-rank factorization of $\bW$}\label{sec:theta_spectrahedron}

We model $\Theta\in \P^d$ as a rank $r$ matrix in the following way: $\Theta=\bU\bU^\top$, where $\bU\in\R^{d\times r}$ and $\|\bU\|_F=1$. The proposed modeling has several benefits in large-scale low-rank matrix learning problems, where $r\ll \min\{d,T\}$ is a common setting. 
{\it First}, the parameterization ensures that $\Theta\in \P^d$ constraint is always satisfied. This saves the costly projection operations to ensure $\Theta\in\P^d$. 
% Existing works that learn $\Theta$ and $\bW$ matrices simultaneously~\citep{Argyriou08,Zhang10,mjaw11a,Ciliberto15} by employing the the variational characterization of the squared trace norm regularizer (Lemma 1) employ $O(d^3)$ SVD computation to learn  $\Theta\in \P^d$. 
Enforcing $\|\bU\|_F=1$ constraint costs $O(rd)$.  
{\it Second}, the dimension of the search space of problem (\ref{eqn:dual11}) with $\Theta=\bU\bU^\top$ is $rd - 1 - r(r-1)/2 $, which is much lower than the dimension ($d(d+1)/2 -1$) of $\Theta \in \P^d$. By restricting the search space for $\Theta$, we gain computational efficiency. 
% As noted earlier, $r\ll d$ is popularly employed in various large-scale applications. 
%{\it Third}, the set of fixed-rank positive semi-definite matrices with unit trace has the structure of a smooth (quotient) manifold~\citep{journee10a}. This allows us to propose highly efficient first- and second-order algorithms for (\ref{eqn:dual11}). 
{\it Third}, increasing the parameter $C$ in  (\ref{eqn:genericPrimal0}) and (\ref{eqn:dual11}) promotes low training error but high rank of the solution, and vice-versa. The proposed fixed-rank parameterization decouples this trade-off. % Hence, we can promote low-rank and low-training error simultaneously. 
%{\it Third}, depending on the regularization parameter $C$, the solution to (\ref{eqn:variational})  and (\ref{eqn:dual11}) is approximately low rank. Putting a smaller value of $C$ gives a lower-rank solution, but at a higher training loss.  The proposed fixed rank parameterization decouples this trade-off and we can promote low-rank along with low-training error simultaneously. } 
% Hence, we observe that the proposed parameterization has several benefits, especially in large scale low-rank ($r\ll d$) matrix learning setting.

\textbf{Remark 3:} With the proposed fixed-rank parameterization of $\Theta$, the expression for primal variable $\bW$ becomes $\bU \bU ^\top (\bZ + \A^*(s))$.

Instead of solving a minimax objective as in (\ref{eqn:dual11}), we solve a minimization problem after incorporating the $\Theta=\bU\bU^\top$ parameterization as follows:
%We re-write (\ref{eqn:dual11}) using the parameterization $\Theta=\bU\bU^\top$  as follows:
%\small
\begin{align}\label{eqn:dual21}
\minop_{\bU\in \R^{d\times r},\|\bU\|_F=1}&\ \ g(\bU),
\end{align}
where the function $g$ is defined as 
\begin{align}
g(\bU) \coloneqq \maxop_{\bZ\in\R^{d\times T},s\in\R^n}\ -CL^{*}(-\bZ/C)-\frac{1}{2}\norm{\bU^\top(\bZ+\A^{*}(s))}_F^2.\label{eqn:dual22}
\end{align}
%It should be noted that instead of solving a minimax objective directly, as in (\ref{eqn:dual11}), we propose to solve an equivalent minimization problem after incorporating the $\Theta=\bU\bU^\top$ parameterization. 

(\ref{eqn:dual21}) is the proposed {\it generic} structured low-rank matrix learning problem. The application-specific details are modeled within the sub-problem (\ref{eqn:dual22}). In Section~\ref{sec:specialized_formulations}, we present specialized versions of (\ref{eqn:dual22}), tailored for applications such as Hankel matrix learning, non-negative matrix completion, and robust-matrix completion. We propose a unified optimization framework for solving (\ref{eqn:dual21}) in Section~\ref{sec:optimizationAlgorithm}. 

%An important outcome of the above modeling is that the {\it expression} for the gradient of $g({\bU})$ in (\ref{eqn:dual21}) is {\it independent} of the application at hand (refer Lemma \ref{lemma:gradient_Hessian}). The application specific information in (\ref{eqn:dual22}) is encoded only through the {\it values} of variables $\bZ$ and $s$, which are used to compute the gradient of $g(\bU)$. This allows the development of a {\it unified} optimization framework for several low-rank matrix learning problems such as Hankel matrix learning, non-negative matrix completion, and robust-matrix completion, to name a few. %Finally, the matrix $\bW$ is learned as $\bU \bU ^\top (\bZ + \A^*(s))$.

%The proposed optimization algorithms are discussed in Section~\ref{sec:optimizationAlgorithm} and the specialized formulations of $g(\bU)$ for various applications are discussed in Section~\ref{sec:specialized_formulations}. 

%The proposed optimization algorithms are discussed in Section~\ref{sec:optimizationAlgorithm} and the specialized formulations of $g(\bU)$ for various applications are discussed in Section~\ref{sec:specialized_formulations}. 

\begin{algorithm*}[tb]
%\begin{multicols}{2}
   \caption{\footnotesize Proposed first- and second-order algorithms for (\ref{eqn:dual21})}
   \label{alg:trust_region}
   {\footnotesize
   \begin{tabular}{ l | l }
  \multicolumn{2}{l}{{\bfseries Input:} matrix $\bY$, rank $r$,  regularization parameter $C$. }   \\
  \multicolumn{2}{l}{ Initialize $\bU\in\mathcal{S}^d_r$.}   \\
  \multicolumn{2}{l}{{\bfseries repeat}} \\
  \multicolumn{2}{l}{\ \ \ \ \ \textbf{1:} Solve for $\{\bZ,s\}$ by computing $g(\bU)$ in~(\ref{eqn:dual22}). Section~\ref{sec:specialized_formulations} discusses solvers for specific applications.  } \\
  \multicolumn{2}{l}{\ \ \ \ \ \textbf{2:} Compute $\nabla g(\bU)$ as given in Lemma \ref{lemma:gradient_Hessian}.}   \vspace{4pt}\\
   \ \ \ \ \ \multirow{3}{190pt}{\textbf{3:}  \textbf{Riemannian CG step:} compute a conjugate direction $\bV$ and step size $\alpha$ using Armijo line search. It makes use of  $\nabla g(\bU)$.} &  \multirow{3}{220pt}{\textbf{3:} \textbf{Riemannian TR step:} compute a search direction $\bV$ which minimizes the trust region sub-problem. It makes use of  $\nabla g(\bU)$ and its directional derivative.  Step size $\alpha=1$. } \\
   \ \ \ \ \ &  \\
   \ \ \ \ \ &  \\
   \ \ \ \ \ &  \vspace{4pt}\\
   \multicolumn{2}{l}{\ \ \ \ \ \textbf{4:} Update $\bU=(\bU+\alpha\bV)/\norm{\bU+\alpha\bV}_F$ (retraction step) } \\
  \multicolumn{2}{l}{{\bfseries until} convergence}   \\%  {\red /*$\|\nabla_{\bU} g(\bU\bU^\top)\|_F\leq \epsilon$*/}
  \multicolumn{2}{l}{{\bfseries Output:} $\{\bU,\bZ,s\}$ and $\bW=\bU\bU^\top(\bZ+\A^{*}(s))$.}
\end{tabular}
}
%\end{multicols}
\end{algorithm*}
% \footnotetext{In standard/robust matrix completion application, SVD based initialization may also be employed~\citep{boumal15a}.}

%\noteBM{More specifically, the proposed optimization algorithm for (\ref{eqn:dual21}), including the gradient $\nabla_{\bU} g$ formula, is independent of the application at hand. The application specific structural or loss-related constraints are encoded only in the value of the gradient $\nabla_{\bU} g$. } The proposed optimization algorithms, including the gradient $\nabla_{\bU} g$ formula, are discussed in Section~\ref{sec:optimizationAlgorithm} and the specialized formulations of $g(\bU)$ for various applications are discussed in Section~\ref{sec:specialized_formulations}. 

%Figure~\ref{fig:solver} presents the overall design of our framework. The outer (global) minimization problem over $\bU$ aims to learn optimal low-dimensional latent space for $\bW$, and the inner (local) maximization problem over $\{s,(z_t)_{t=1}^T\}$ learns the best $\bW$ in that space that satisfy the structural constraints $\A(\bW)=\bzero$. 
%\subsection{Duality gap certificate and global optimality}
The fixed-rank parameterization, $\Theta=\bU\bU^\top$, results in non-convexity of the overall optimization problem (\ref{eqn:dual21}), though sub-problem (\ref{eqn:dual22}) is a convex optimization problem. 
We end this section by stating sufficient conditions of obtaining a globally optimal solution of (\ref{eqn:dual11}) from a solution of (\ref{eqn:dual21}). 
%We end this section by providing the duality gap optimality criterion to verify whether a {\it feasible} solution $\hat{\bU}$ of (\ref{eqn:dual21}) is optimal with respect to (\ref{eqn:dual11}). 
%In this regard, it should be noted that though (\ref{eqn:dual21}) is a non-convex problem in $\bU$, the optimization problem in (\ref{eqn:dual22}) is convex in $\{\bZ,s\}$ for a given $\bU$. 
\begin{theorem}\label{prop:dualityGap}
Let $\hat{\bU}$ be a feasible solution of~(\ref{eqn:dual21}) and $\{\hat{\bZ},\hat{s}\}$ be an optimal solution of the convex  problem in (\ref{eqn:dual22}) at $\bU=\hat{\bU}$. Let $\sigma_1$ be the maximum singular of the  matrix $\hat{\bZ}+\A^{*}(\hat{s})$. A candidate solution for  (\ref{eqn:dual11}) is $\{\hat{\Theta},\hat{\bZ},\hat{s}\}$, where $\hat{\Theta}=\hat{\bU}\hat{\bU}^\top$. 
The duality gap ($\Delta$) associated with $\{\hat{\Theta},\hat{\bZ},\hat{s}\}$ is given by 
%{\footnotesize$\Delta = \frac{1}{2}(\sigma_1^2 - \sum_{t=1}^T\norm{\hat{\bU}^\top(\hat{z_t}+\hat{a_t})}^2).$}
\begin{equation}
 \Delta = \frac{1}{2}\Big(\sigma_1^2 - \norm{\hat{\bU}^\top(\hat{\bZ}+\A^{*}(\hat{s}))}_F^2\Big).\nonumber
\end{equation}
Furthermore, if $\hat{\bU}$ is a rank deficient local minimum of~(\ref{eqn:dual21}), then $\{\hat{\Theta},\hat{\bZ},\hat{s}\}$ is a global minimum of (\ref{eqn:dual11}), i.e., $\Delta=0$. 
\end{theorem}
\begin{proof}
The min-max problem (\ref{eqn:dual11}) can be equivalently re-written as 
\begin{equation}
\minop_{\Theta\in\P^d} G_1(\Theta),
\end{equation}
where 
\begin{equation}
G_1(\Theta)\coloneqq \maxop_{s\in\R^n}\maxop_{\bZ\in\R^{d\times T}}  - CL^{*}(\frac{-\bZ}{C}) - \frac{1}{2}\inner{\Theta(\bZ+\A^{*}(s)),\bZ+\A^{*}(s)}.
\end{equation}
Given $\{\hat{\Theta}=\hat{\bU}\hat{\bU}^\top\}$ as described in the statement of the theorem, $G_1(\hat{\Theta})=f(\hat{\Theta},\hat{\bZ},\hat{s})$. 

Using the min-max interchange~\citep{Sion58}, the max-min problem equivalent to (\ref{eqn:dual11}) can be written as  
\begin{equation}
\maxop_{s\in\R^n}\maxop_{\bZ\in\R^{d\times T}}  G_2(\bZ,s),
\end{equation}
where
\begin{align}
&  G_2(\bZ,s)\coloneqq -CL^{*}(\frac{-\hat{\bZ}}{C})  - \frac{1}{2}B(\bZ,s),\ \textup{and}\\
 & B(\bZ,s)=\maxop_{\Theta\in \P^d}\inner{{\Theta},({\bZ}+\A^{*}({s}))({\bZ}+\A^{*}({s}))^\top}.\label{appendix:eqn:traceDual}
\end{align}
Note that problem (\ref{appendix:eqn:traceDual}) is a well studied problem in the duality theory. It is one of the definitions of the spectral norm (maximum eigenvalue of a matrix) -- as the dual of the trace norm~\citep{Boyd04}. Its optimal value is the spectral norm of the matrix $({\bZ}+\A^{*}({s}))({\bZ}+\A^{*}({s}))^\top$~\citep{Boyd04}. 

Let $\sigma_1$ be the maximum singular of the matrix $(\hat{\bZ}+\A^{*}(\hat{s}))$. Then, the spectral norm of a symmetric matrix $(\hat{\bZ}+\A^{*}(\hat{s}))(\hat{\bZ}+\A^{*}(\hat{s}))^\top$ is equal to $\sigma_1^2$.  

The duality gap ($\Delta$) associated with $\{\hat{\Theta},\hat{\bZ},\hat{s}\}$   is given by 
\begin{align}
\Delta&=G_1(\hat{\Theta}) - G_2(\hat{\bZ},\hat{s})\nonumber \\
 &=\frac{1}{2}\Big(\sigma_1^2 - \norm{\hat{\bU}^\top(\hat{\bZ}+\A^{*}(\hat{s}))}_F^2\Big).\nonumber
\end{align}

Last part of the theorem: it is a special case of the general result proved by~\citet[Theorem~7~and~Corollary 8]{journee10a}. 

\end{proof}

The value of the duality gap $\Delta$ can be used as  to verify whether a candidate solution $\{\hat{\Theta},\hat{\bZ},\hat{s}\}$ is a global optimum of (\ref{eqn:dual11}). 
The cost of computing $\sigma_1$ is computationally cheap as it requires only a few {\it power iteration} updates.% and the term $\norm{\hat{\bU}^\top(\hat{\bZ}+\A^{*}(\hat{s}))}_F^2$ is computed while optimizing $g(\bU)$ (discussed in Section~\ref{sec:specialized_formulations}). 

\section{Optimization on spectrahedron manifold}\label{sec:optimizationAlgorithm}
The matrix $\bU$ lies in, what is popularly known as, the \emph{spectrahedron} manifold $\mathcal{S}^d_r \coloneqq \{ \bU \in \mathbb{R}^{d\times r}: \|\bU \|_F = 1 \}$. Specifically, the spectrahedron manifold has the structure of a compact Riemannian quotient manifold \citep{journee10a}. The quotient structure takes the rotational invariance of the constraint $\|\bU \|_F = 1$ into account. The Riemannian optimization framework embeds the constraint $\bU\in\mathcal{S}^d_r $ into the search space, conceptually translating the constrained optimization problem (\ref{eqn:dual21}) into an {\it unconstrained} optimization over the spectrahedron manifold. The Riemannian optimization framework generalizes various classical first- and second-order Euclidean algorithms (e.g., the conjugate gradient  and trust region algorithms) to manifolds and provide concrete convergence guarantees \citep{edelman98a,absil08a,journee10a,sato13a,sato17a}. In particular, \citet{absil08a} provide a systematic way of implementing Riemannian conjugate gradient (CG) and trust region (TR) algorithms. 
The particular details are provided in Appendix \ref{sec:optSpectrahedron}. 

We implement the  Riemannian conjugate gradient (CG) and trust-region (TR) algorithms for (\ref{eqn:dual21}). These require the notions of the {\it Riemannian gradient} (first-order derivative of the objective function on the manifold), {\it Riemannian Hessian} along a search direction (the \emph{covariant} derivative of the Riemannian gradient along a tangential direction on the manifold), and the {\it retraction} operator (that ensures that we always stay on the manifold). 
%Once these notions are defined, we can employ the publicly available toolboxes, {\it e.g.}, Manopt \citep{boumal14a}, for an efficient implementation. 
The Riemannian gradient and Hessian notions require computations of the standard (Euclidean) gradient  and the directional derivative of this gradient along a given search direction, which are expressed in the following lemma. 
%The following lemma shows the concrete matrix expressions of the gradient and its directional derivative.
%At step $k$ of the algorithm, given $\bU_{k-1}$ matrix obtained from step $k-1$, we require to compute $\nabla{ g(\bU\bU^\top)}|_{\bU=\bU_{k-1}}$. 
\begin{lemma}\label{lemma:gradient_Hessian}
Let $\{\hat{\bZ},\hat{s}\}$ be an optimal solution of the convex problem (\ref{eqn:dual22}) at $\bU$. 
Then, the gradient of $g(\bU)$ at $\bU$ is given by the following expression: 
\begin{equation}
\nabla {g}(\bU)  = -  (\hat{\bZ} + \A^{*}(\hat{s}))(\hat{\bZ} + \A^{*}(\hat{s}))^\top\bU.\nonumber
\end{equation}
Let  ${\rm D} \nabla g (\bU) [\bV] $ denote the directional derivative of the gradient $\nabla {g}(\bU)$ along $\bV \in \R^{d\times r}$. 
Let $\{\dot{\bZ},\dot{s}\}$ denote the directional derivative of $\{\bZ,s\}$ along $\bV$ at $\{\hat{\bZ},\hat{s}\}$. 
Then, 
\begin{align}
{\rm D} \nabla g (\bU) [\bV] = (\dot{\bZ} + \A^{*}(\dot{s}))(\hat{\bZ} + \A^{*}(\hat{s}))^\top\bU +  (\hat{\bZ} + \A^{*}(\hat{s}))((\dot{\bZ} + \A^{*}(\dot{s}))^\top\bU- (\hat{\bZ} + \A^{*}(\hat{s}))^\top\bV).\nonumber
\end{align}
\end{lemma}
\begin{proof}
 The gradient is computed by employing the Danskin's theorem~\citep{Bertsekas99,Bonnans00}. The directional derivative of the gradient follows directly from the chain rule. 
\end{proof}
The terms $\{\dot{\bZ},\dot{s}\}$ are computed from the first-order KKT conditions of the convex problem (\ref{eqn:dual22}) at $\{\hat{\bZ},\hat{s}\}$. In particular, when $L^{*}(\cdot)$ is differentiable (e.g., the square loss), $\{\dot{\bZ},\dot{s}\}$ are obtained by solving the following linear system:
{\small
\begin{align*}
 &{\rm D} \nabla L^{*}(-\bZ/C)[\dot{\bZ}] - (\bU\bV^\top + \bV\bU^\top)(\hat{\bZ}+\A^{*}(\hat{s})) - \bU\bU^\top(\dot{\bZ}+\A^{*}(\dot{s}))  = \bzero\\
&\A\big((\bU\bV^\top + \bV\bU^\top)(\hat{\bZ}+\A^{*}(\hat{s})) + \bU\bU^\top(\dot{\bZ}+\A^{*}(\dot{s}))\big)=\bzero.
\end{align*}
}In various applications such as the traditional matrix completion or multi-variate regression, both the gradient $\nabla { g}(\bU)$ and its directional derivation ${\rm D} \nabla g (\bU) [\bV]$ can be computed by closed-form expressions.

\textbf{Riemannian CG algorithm:} It computes the Riemannian \emph{conjugate} gradient direction by employing the first-order information $\nabla { g}(\bU)$ (Lemma~\ref{lemma:gradient_Hessian}). We perform \emph{Armijo} line search on $\mathcal{S}^d_r$ to compute a step-size that sufficiently decreases $g(\bU)$ on the manifold. We update along the conjugate direction with the step-size by {\it retraction}.

\textbf{Riemannian TR algorithm:} It solves a Riemannian trust-region \emph{sub-problem} (in a neighborhood) at every iteration. Solving the trust-region sub-problem leads to a search direction that minimizes a {\it quadratic} approximation of $g(\bU)$ on the manifold. Solving this sub-problem does not require inverting the full Hessian of the objective function. It makes use of $\nabla { g}(\bU)$ and its directional derivative  ${\rm D} \nabla g (\bU) [\bV]$ (Lemma~\ref{lemma:gradient_Hessian}). 

\textbf{Overall algorithm:} Algorithm \ref{alg:trust_region} summarizes the proposed first- and second-order algorithms for solving (\ref{eqn:dual21}).  %Algorithm~\ref{alg:trust_region} terminates when $\|\nabla_{\bU} g(\bU)\|_F\leq \epsilon$, where $\epsilon$ is a pre-defined tolerance threshold. In large-scale experiments, we follow the common practice of also having an upper limit on the number of iterations~\citep{toh10a,Hsieh14,Cambier16,wen12a,mishra14c,yu14b,boumal15a,tan16a}. 

\textbf{Convergence:} \citet{absil08a,sato13a} discuss the rate of convergence analysis of manifold algorithms, and their results are directly applicable in the present setting. The global convergence results for conjugate gradient algorithm is discussed in \citep{sato13a}. For trust regions, the global convergence to a first-order critical point is discussed in \citep{absil08a}[Section~7.4.1], while the local convergence to local minima is discussed in \citep{absil08a}[Section~7.4.2]. 

\textbf{Computational complexity:} The computational complexity of Algorithm \ref{alg:trust_region} is the sum of the cost of manifold related operations and the cost of application specific ingredients. The spectrahedron manifold operations cost $O(dr + r^3)$. The following section discusses the application specific computational costs.

Although we have focused on batch algorithms, our framework can be extended to stochastic settings, {\it e.g.}, when the columns are streamed one by one \citep{bonnabel13a}. In this case, when a new column is received, we perform a (stochastic) gradient update on $\mathcal{S}_r^d$. 

%It can be observed from Lemma~\ref{lemma:gradient_Hessian} that the gradient $\nabla_{\bU} g(\bU)$ depends on an optimal solution of the convex problem~(\ref{eqn:dual22}). Section~\ref{sec:specialized_formulations} discusses efficient solvers for solving (\ref{eqn:dual22}) for  various applications. \noteBM{For example, in matrix completion and multi-task feature learning problems, (\ref{eqn:dual22}) can be solved in closed form. }

% We provide sufficient criterion for global optimality and the duality gap criterion of the solution obtained by Algorithm \ref{alg:trust_region} with respect to  (\ref{eqn:dual11}) in the next section. 

%Problem~(\ref{eqn:dual21}) is on the set of $d\times r$ matrices with unit Frobenius norm, defined as,
%\begin{equation*}
%\begin{array}{lll}
%\mathcal{S}^d_r \coloneqq \{ \bU \in \mathbb{R}^{d\times r}: \|\bU \|_F = 1 \}.
%\end{array}
%\end{equation*}

\section{Specialized formulations for applications} \label{sec:specialized_formulations}
%The overall problem (\ref{eqn:dual21}) remains the same for various structured low-rank matrix learning problems. 
The expression of $g(\bU)$ in (\ref{eqn:dual22}) depends on the functions $L(\cdot)$ and $\A(\cdot)$ employed in the application at hand. Below, we discuss $g(\bU)$ for popular applications. 
%Depending on the application specific constraints, the function $g(\bU)$ changes. 
%Below we discuss the concrete formulations for popular applications. 

\subsection{Matrix completion}\label{sec:mcformulation}
Given a partially observed matrix $\bY$ at indices $\Omega$, we learn the full matrix $\bW$ \citep{toh10a,cai10a}. Let $\Omega_t$ be the set of indices that are observed in $y_t$, the $t^{th}$ column of $\bY$. Let $y_{t_{\Omega_t}}$ and $\bU_{\Omega_t}$ represents the rows of $y_t$ and $\bU$, respectively, that correspond to the indices in $\Omega_t$. Then, the function $g(\bU)$ in (\ref{eqn:dual22}), for the standard low-rank matrix completion problem with square loss, is as follows:
%\begin{align*}
%g(\bU) = \maxop_{\bZ\in\R^{d\times T}}\inner{\bY_{\Omega},\bZ_{\Omega}}-\frac{\norm{\bZ_{\Omega}}_F^2}{4C}-\frac{1}{2}\norm{\bU^\top\bZ}_F^2
%\end{align*}
\begin{align}\label{eqn:mc}
g(\bU) = \sum_{t=1}^T\maxop_{z_t\in\R^{|\Omega_t|}}\ \inner{y_{t_{\Omega_t}},z_t} -\frac{1}{4C}\norm{z_t}^2-\frac{1}{2}\norm{\bU_{\Omega_t}^\top z_t}^2. 
\end{align}
Problem (\ref{eqn:mc}) is a least-squares problem for each $z_t$ and can be solved efficiently in closed-form by employing the Woodbury matrix-inversion identity. The computational cost of solving problem (\ref{eqn:mc}) is $O(|\Omega| r^2)$. Overall, the per-iteration computational cost of Algorithm \ref{alg:trust_region} is $O(|\Omega| r^2 + dr + r^3 )$. Problem (\ref{eqn:mc}) can be solved in parallel for each $t$, making it amenable to parallelization.

%In addition, (\ref{eqn:mc}) can be solved in parallel for each $z_t$. 

\subsection{Robust matrix completion}\label{sec:rmcformulation}
The setting of this problem is same as the low-rank matrix completion problem. Following  \citet{Cambier16}, we employ a {\it robust} loss function ($\ell_1$-loss) instead of the square loss. 
% except that a {\it robust} loss function is employed instead of the square loss \citep{Cambier16}. 
%Due to the non-smooth nature of the $\ell_1$-loss, a recently developed state-of-the-art large scale robust matrix completion algorithm~\citep{Cambier16} employs huber-loss, which is a smooth approximation of the $\ell_1$-loss. 
%We employ the $\ell_1$-loss for robust low-rank matrix completion problem. 
The expression for $g(\bU)$, in our case, is as follows: 
%In this setting, $g(\bU)$ can be derived as 
\begin{equation*}%\label{eqn:rmc}
	g(\bU) = \sum_{t=1}^T\ \maxop_{z_t\in[-C,C]^{|\Omega_t|}}\ \inner{y_{t_{\Omega_t}},z_t} -\frac{1}{2}\norm{\bU_{\Omega_t}^\top z_t}^2. 
	%{\color{red} g(\bU) = \sum_{t=1}^T\ \maxop_{z_{t_{\Omega_t}}\in[-C,C]^{|\Omega_t|}}z_{t_{\Omega_t}}^\top y_{t_{\Omega_t}}-\frac{1}{2}\norm{\bU_{\Omega_t}^\top z_{t_{\Omega_t}}}^2.}
\end{equation*}
%Coordinate descent algorithm is employed to efficiently solve (\ref{eqn:rmc}). 
Coordinate descent algorithm (CD) is employed to efficiently solve the above problem. The computational cost depends on the number of iterations of the coordinate descent algorithm. In particular, if $k$ is the number of iterations, then the cost of computing $g(\bU)$ is $O(|\Omega|kr^2)$. 
A small value of $k$ suffices for a good approximation to the solution. 
%In our experiments, we observe that a small value of $k$ is sufficient to obtain a good approximation of (\ref{eqn:rmc}). 
%{\color{red} The gradient of the function $g(\bU)$ is given by the expression $\nabla_{\bU} g(\bU)=-\sum_{t=1}^T\ z_t(z_{t_{\Omega_t}}^\top\bU_{\Omega_t})$. }
Hence, for the case of robust matrix completion problem, the per-iteration computational cost of Algorithm \ref{alg:trust_region} is $O(|\Omega| kr^2 + dr + r^3 )$. It should be noted that the computation of $g(\bU)$ can be a parallelized across $t = \{1, \ldots, T \}$.
In addition to the $\ell_1$-loss, we also experimented with the $\epsilon$-SVR loss \citep{Ho12} for this problem.

\subsection{Non-negative matrix completion}\label{sec:nnmcformulation}
Given a partially observed matrix $\bY$ at indices $\Omega$, the aim here is to learn the full matrix $\bW$ with non-negative entries only. 
%Let $\Omega_t$ be the set of indices that are observed in $y_t$, the $t^{th}$ column of $\bY$. Let $y_{t_{\Omega_t}}$ and $\bU_{\Omega_t}$ represents the rows of $y_t$ and $\bU$, respectively, that correspond to the indices in $\Omega_t$. 
For this problem, the function $g(\bU)$ with the square loss function is as follows:
\begin{align}\label{eqn:nnmc}
g(\bU) = \sum_{t=1}^T\ \maxop_{s_t\in[0,\infty)^d}\ \bigg(\maxop_{z_t\in\R^{|\Omega_t|}}\  \inner{y_{t_{\Omega_t}},z_t} -\frac{1}{4C}\norm{z_t}^2-\frac{1}{2}\norm{\bU_{\Omega_t}^\top z_t+\bU^\top s_t}^2\bigg).
\end{align}
Since the dual variables $s_t$ have non-negative constraints (due to the constraint $\bW\geq0$ in primal problem), we model (\ref{eqn:nnmc}) as a non-negative least squares (NNLS) problem. 
We employ the NNLS algorithm of \citet{kim13a} to solve for the variable $s_t$ in (\ref{eqn:nnmc}). In each iteration of NNLS, $z_t$ is solved in closed form. 
If $k$ is the number of iterations of NNLS, then the cost of computing $g(\bU)$ is $O(dTkr + |\Omega| k r^2)$. Hence, for the case of non-negative matrix completion problem, the per-iteration computational cost of Algorithm \ref{alg:trust_region} is $O(dT kr + |\Omega| k r^2+ dr + r^3 )$. In our experiments, we observe that a small value of $k$ is sufficient to obtain a good approximation of (\ref{eqn:nnmc}). 

%Similar to Section \ref{sec:rmcformulation}, a small value of $k$ suffices for a good approximation to the solution of (\ref{eqn:nnmc}).

It should be noted that (\ref{eqn:nnmc}) is computationally challenging as it has $dT$ entry-wise non-negativity constraints. In our initial experiments, we observe that the solution $[s_1,\ldots,s_T]$ is highly sparse. For large-scale problems, we exploit this observation for an efficient implementation.

%These entries are randomly selected every time we compute $g(\bU)$. 

%On the other hand, the constraints in (\ref{eqn:nnmc}) are separable across $t$, which can be imposed in parallel.

\subsection{Hankel matrix learning}\label{sec:hmlformulation}
Hankel matrices have the structural constraint that its anti-diagonal entries are the same. A Hankel matrix corresponding to a vector $y=[y_1,y_2,\ldots,y_7]$ is as follows: 
%\newline
{%\scriptsize
\begin{equation*}\label{eqn:HankelMatrixExample}
\left[ 
\begin{array}{ccccc}
y_1 & y_2 & y_3 & y_4 & y_5\\
y_2 & y_3 & y_4 & y_5 & y_6\\
y_3 & y_4 & y_5 & y_6 & y_7
\end{array}
\right].
\end{equation*}}%
Hankel matrices play an important role in determining the order or complexity of various linear time-invariant systems~\citep{Fazel13,markovsky13a}. The aim in such settings is to find the minimum order of the system that explains the observed data (i.e. low error). A low-order system is usually desired as it translates into cost benefit and ease of analysis. Hence, the optimization problem can be modeled as problem (\ref{eqn:genericPrimal0}) where we wish to learn a low-rank Hankel matrix that results in low error. 
%Section~\ref{sec:empiricalResults} discusses one such application, stochastic system realization, that require learning low-rank Hankel matrices. 
%Given the observation of noisy system output, the goal in stochastic system realization (SSR) problem is to find a minimal order autoregressive moving-average model~\citep{Fazel13,yu14b}. The order of such a model is equal to the rank of the Hankel matrix consisting of the exact process covariances. 
%Hence, finding a low-order model is equivalent to learning a low-rank Hankel matrix close to the given data $y$ \citep{Fazel13,markovsky13a,markovsky14a}. 

%In certain linear time-invariant systems and stochastic system realization problems, learning a low-rank Hankel matrix corresponding to a given vector $y$ is equivalent to finding a low-order linear model for the data \citep{Fazel13,markovsky13a,markovsky14a}. \noteBM{Need to add some more explanation regarding Hankel matrix utility - from Fazel}. 

The function $g(\bU)$ specialized for the case of low-rank Hankel matrix learning with square loss is as follows: 
\begin{align}
g(\bU) = \maxop_{s_t\in\R^{d}\forall t,z\in\R^{d+T-1}} z^\top y - \frac{1}{4C}\norm{z}^2 - \frac{1}{2}\sum_{t=1}^T\norm{\bU^\top s_t}^2\nonumber\\
{\rm subject\ to:}\sum_{\substack{(i,t):i+t=k,\\1\leq i\leq d,1\leq t\leq T}}s_{ti}=z_k\ \forall k=2,\ldots,d+T\nonumber.
\end{align}
We solve the above problem via a conjugate gradient algorithm. The equality constraints are handled efficiently by using an affine projection operator. The computational cost of computing $g(\bU)$ is $O(dTkr)$, where $k$ is the number of iterations of the inner conjugate gradient algorithm. Hence, for the case of Hankel matrix learning problem, the per-iteration complexity of Algorithm \ref{alg:trust_region} is $O(dT kr +  dr + r^3 )$.

%\subsection{Other applications}
%Applications such as the traditional matrix completion~\citep{boumal11a,mishra14c}, multi-task feature learning~\citep{Argyriou08} or multivariate regression~\citep{yuan07a} aim at learning low-rank matrices with square loss and without any structural constraint. These are special cases of problem (\ref{eqn:genericPrimal0}), without any $\bW\in\D$ constraint, and can be easily solved under the proposed framework. %Specialized formulations and empirical results for such applications are discussed in the supplementary material. 

\subsection{Multi-task feature learning}\label{sec:mtflformulation}
The paradigm of multi-task feature learning advocates learning several {\it related} tasks (regression/classification problems) jointly by sharing a low-dimensional latent feature representation across all the tasks \citep{Argyriou08}. Since the tasks are related, learning them together, by sharing {\it knowledge}, is expected to obtain better generalization than learning them independently. 

Multi-task feature learning setup is as follows: we are given $T$ tasks and the aim is to learn the model parameter $w_t$ for each task $t$ such that they share a low-dimensional latent feature space. \citet{Argyriou08} employ the trace-norm regularizer  on the model parameter matrix $\bW=[w_1,\ldots,w_T]$ to enforce a low-dimensional representation of $\bW$. Each task $t$ has an input/output training data set  $\{\bX_t,y_t\}$, where $\bX_t \in \mathbb{R}^{n_t \times d}$ and $y_t \in \mathbb{R}^{n_t}$. The prediction function for task $t$ is $h_t(x)=\inner{x,w_t}$. 
%The square loss function for MTFL can be expressed as $L(\bW,\bY;\{\bX_1,\ldots,\bX_T\})=\sum_{t=1}^T \norm{y_t - \bX_t w_t}^2 $. 

The proposed generic formulation for structured low-rank matrix learning (\ref{eqn:dual21}) can easily be specialized to this setting, with the function $g(\bU)$ being as follows:
\begin{align}\label{eqn:mtl}
g(\bU) = \sum_{t=1}^T \maxop_{ z_t  \in \R^{n_t}}\ \inner{y_t,z_t} - \frac{1}{4C}\norm{z_t}^2 - \frac{1}{2}\norm{\bU^\top\bX_t^\top z_t}^2\nonumber. 
\end{align}
The function $g(\bU)$  can be computed in closed form and costs $O(r(r+d)(\sum_t n_t) + r^3T)$. Hence, the per-iteration complexity of Algorithm \ref{alg:trust_region} for MTFL is $O(r(r+d)(\sum_t n_t) + r^3T +  dr + r^3 )$. The computation of $g(\bU)$ can be a parallelized across the tasks. This setting can be further specialized for problems such as matrix completion with side information, inductive matrix completion, and multi-variate regression.

\section{Experiments}\label{sec:empiricalResults}
In this section, we evaluate the generalization performance as well as computational efficiency of our approach against state-of-the-art in different applications. It should be emphasized that state-of-the-art in each application are different and to the best of our knowledge there does not exist a unified framework for solving such applications. 
All our algorithms are implemented using the Manopt toolbox~\citep{boumal14a}. We term our algorithm as \textbf{R}iemannian \textbf{S}tructured \textbf{L}ow-rank \textbf{M}atrix learning (\textbf{RSLM}). 
\begin{table}\centering
{\small
\centering
\caption{Data set statistics}\label{table:datasetnnmc}
\begin{tabular}{llll}
\toprule
Data set  & \multicolumn{1}{c}{$d$} & \multicolumn{1}{c}{$T$} & \multicolumn{1}{c}{$|\Omega|$}\\
\midrule 
MovieLens1m (ML1m) & $3\,706$ & $6\,040$ & $1\,000\,209$\\
MovieLens10m (ML10m) & $10\, 677$ & $71\, 567$ & $10\, 000\, 054$ \\
MovieLens20m (ML20m) & $26\, 744$ & $138\, 493$ & $20\, 000\, 263$ \\
Netflix & $17\, 770$ & $480\, 189$ & $100\, 198\, 805$\\
Synthetic & $5\,000$ & $500\,000$ & $15\,149\,850$\\
\bottomrule
\end{tabular}
}
\end{table}
\subsection{Matrix Completion}
\textbf{Baseline techniques:} 
Our first- and second-order matrix completion algorithms are denoted by RSLM-cg and RSLM-tr, respectively. We compare against state-of-the-art fixed-rank and nuclear norm minimization based matrix completion solvers: 
\begin{itemize}%[noitemsep,nolistsep]
  \item APGL: An accelerated proximal gradient approach for nuclear norm regularization with square loss function~\citep{toh10a}.
  \item Active ALT: State-of-the-art first-order  nuclear norm solver based on active subspace selection~\citep{Hsieh14}. 
  \item MMBS: A second-order fixed rank nuclear norm minimization algorithm~\citep{mishra13c}. It employs an efficient factorization of the matrix $\bW$ which renders the trace norm regularizer differentiable in the primal formulation. %A second-order trust region algorithm is employed.  
  \item R3MC: A non-linear conjugate gradient based approach for fixed rank matrix completion~\citep{mishra14c}. It employs a Riemannian  preconditioning technique, customized for the square loss function. 
  \item RTRMC: It models fixed rank matrix completion problems with square loss on the Grassmann manifold and solves it via a second order preconditioned Riemannian trust-region method~\citep{boumal11a,boumal15a}. 
  \item LMaFit: A nonlinear successive over-relaxation based approach for low rank matrix completion based on alternate least squares~\citep{wen12a}. 
  \item PRP: a recent proximal Riemannian pursuit algorithm~\cite{tan16a}. 
\end{itemize}

% The Netflix data set~\citep{recht13a} has $d=17\, 770$, $T=480\, 189$ and $100\, 198\, 805$ revealed entries. 
%Following~\citet{boumal15a}, we center the data around the mean of the training data (around $3.xx$ across splits). 

%\textbf{Parameter settings:} 
\textbf{Data sets and experimental setup:} 
We compare the performance of the above algorithms on a synthetic data set and three movie recommendation data sets: Netflix \citep{recht13a}, MovieLens10m (ML10m), and MovieLens20m (ML20m) \citep{harper15a}. The data set statistics is provided in Table \ref{table:datasetnnmc}. The regularization parameters for respective algorithms are cross-validated in the set $\{1e-5,1e-4,\ldots,1e5\}$ to obtain their best generalization performance. The optimization strategies for the competing algorithm were set to those prescribed by their authors. For instance, line-search, continuation and truncation were kept on for APGL. The initialization for all the algorithms is based on the first few singular vectors of the given partially complete matrix $\bY$ \citep{boumal15a}. 
All the fixed algorithms (R3MC, LMaFit, MMBS, RTRMC, Proposed-cg-sq, Proposed-tr-sq) are provided the rank $r=10$ for real data sets and $r=5$ for synthetic data set. 
In all variable rank approaches (APGL, Active ALT, PRP),  the maximum rank parameter is set to $10$ for real data sets and $5$ for synthetic data set. 
% For synthetic data sets, $r$ is set to $5$ and for real world data sets, $r$ is set to $10$. 
We run all the methods on five random $80$/$20$ train/test splits and report the average root mean squared error on the test set (test RMSE).  We report the minimum test RMSE achieved after the algorithms have converged or have reached maximum number of iterations. % (\texttt{maxIter}). 
%For first-order methods, \texttt{maxIter} is set to $500$ for the Netflix data set and $200$ for other smaller data sets. For second-order methods, \texttt{maxIter} is set to $100$ for the Netflix data set and $60$ for other smaller data sets. 

\begin{figure*}\centering
%\vspace*{0.75em}
{
\includegraphics[width=0.47\columnwidth]{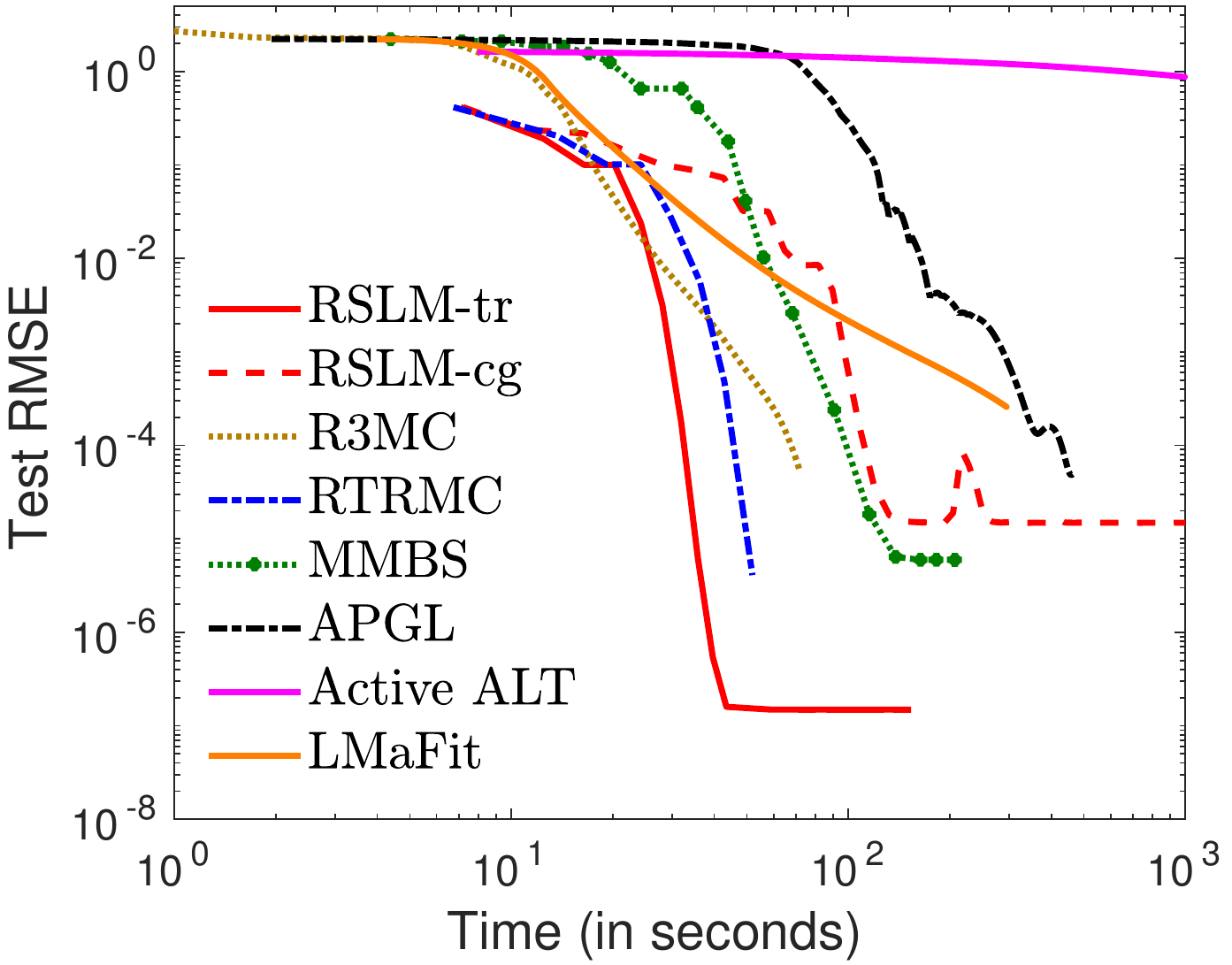}
\hspace*{\fill}
\includegraphics[width=0.47\columnwidth]{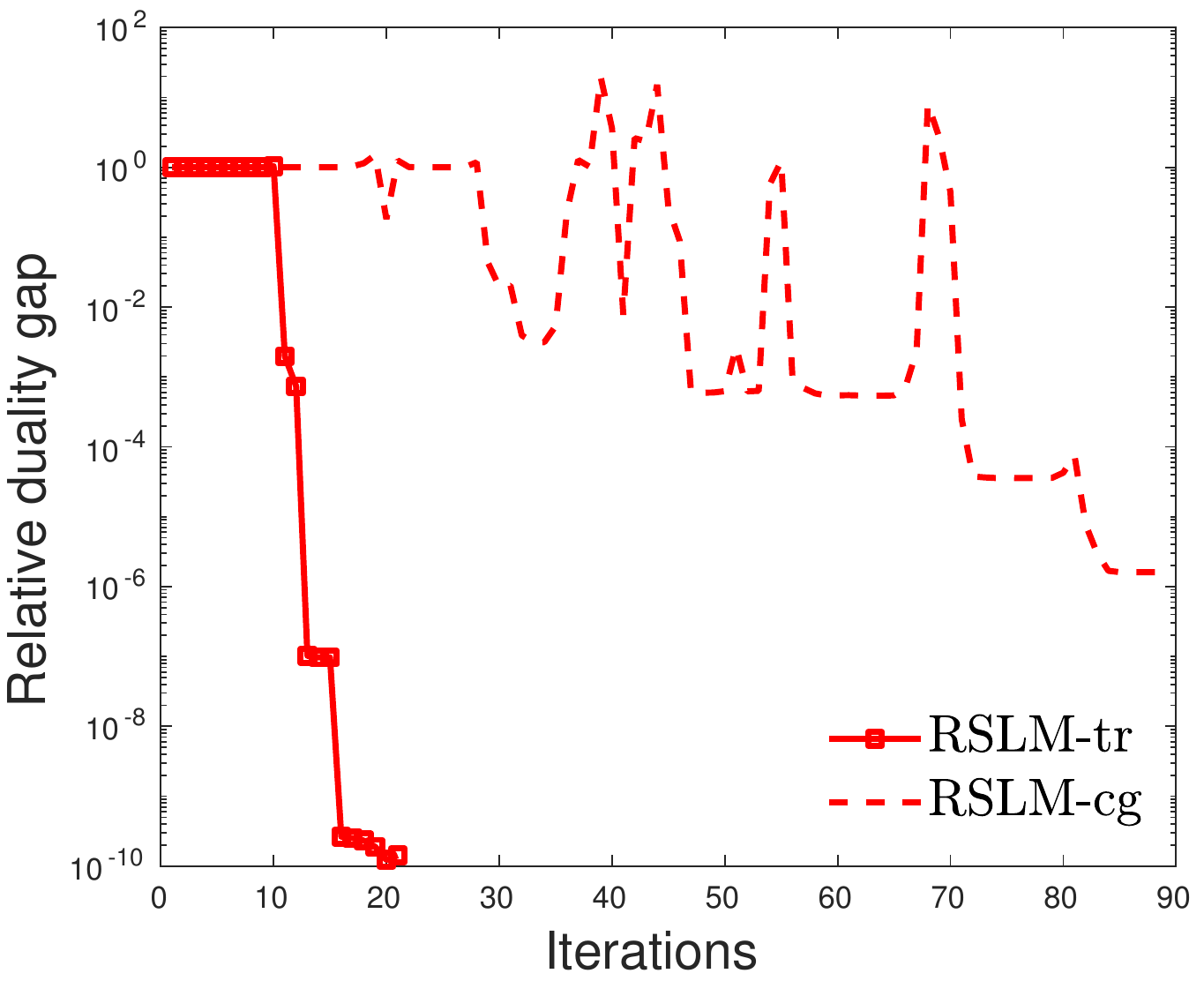}\\
}
\TwoLabels {(a)} {(b)}
\caption{%
(a) Evolution of test RMSE on the synthetic data set. Both our methods obtain very low test RMSE; (b)  Variation of  the relative duality gap per iteration for our methods on the synthetic data set. It can be observed that both our algorithms attain low relative duality gap. 
}\label{appendix:fig:multipleMatrixCompletionFigs}
\end{figure*}
\begin{figure*}[t]\centering
    \includegraphics[width=0.32\columnwidth]{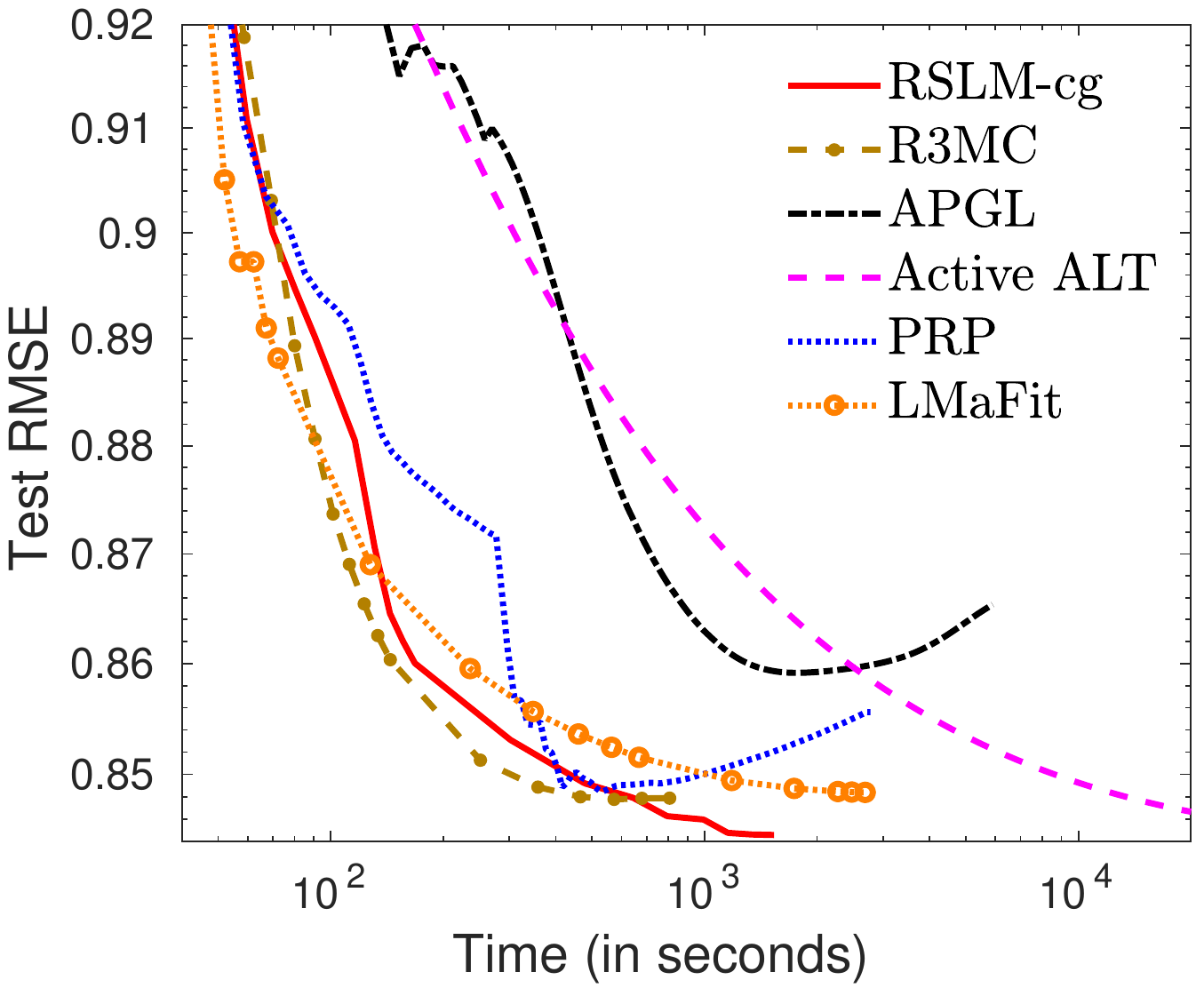}
%    \hspace*{\fill}
    \includegraphics[width=0.32\columnwidth]{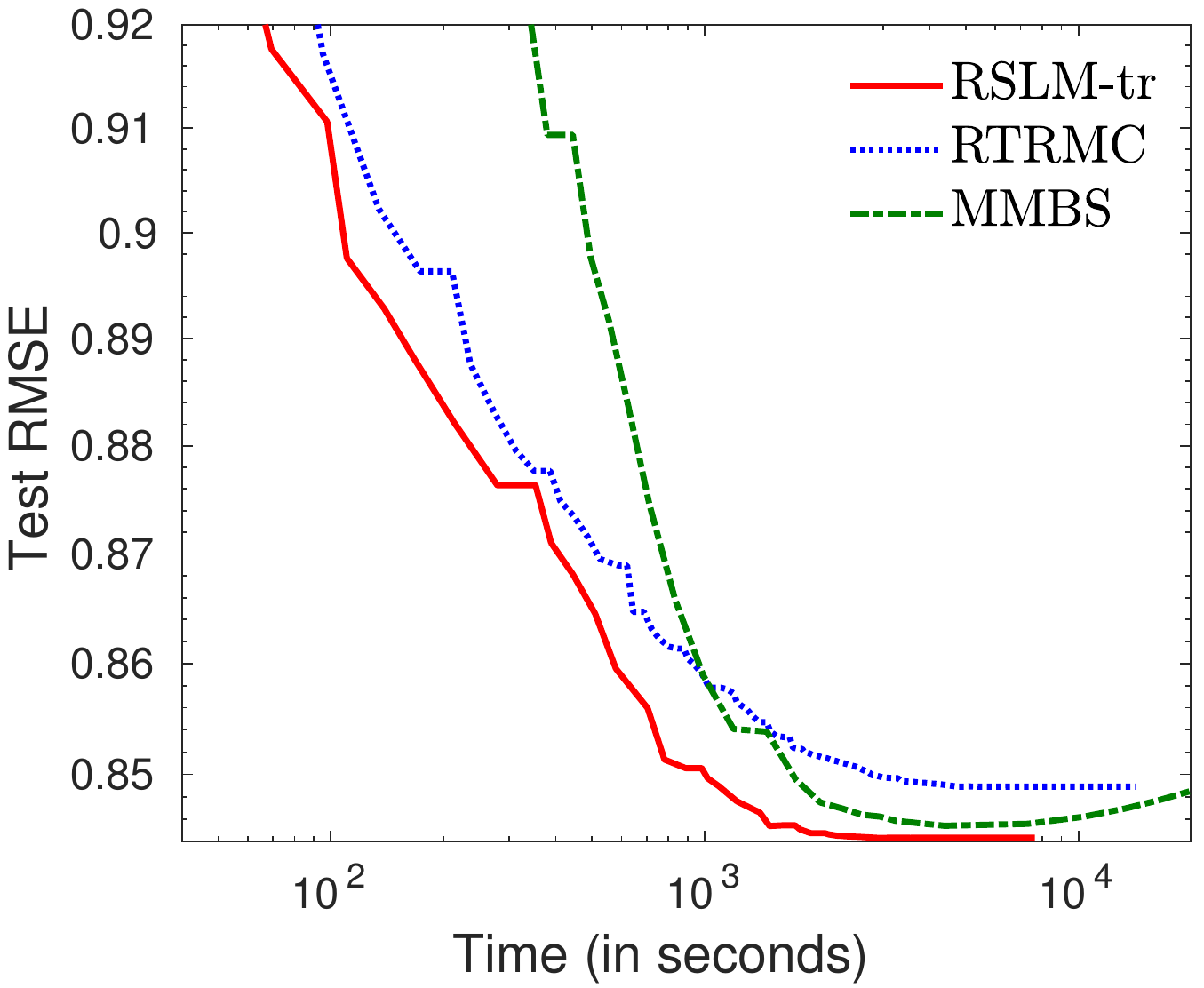}
%    \hspace*{\fill}
    \includegraphics[width=0.33\columnwidth]{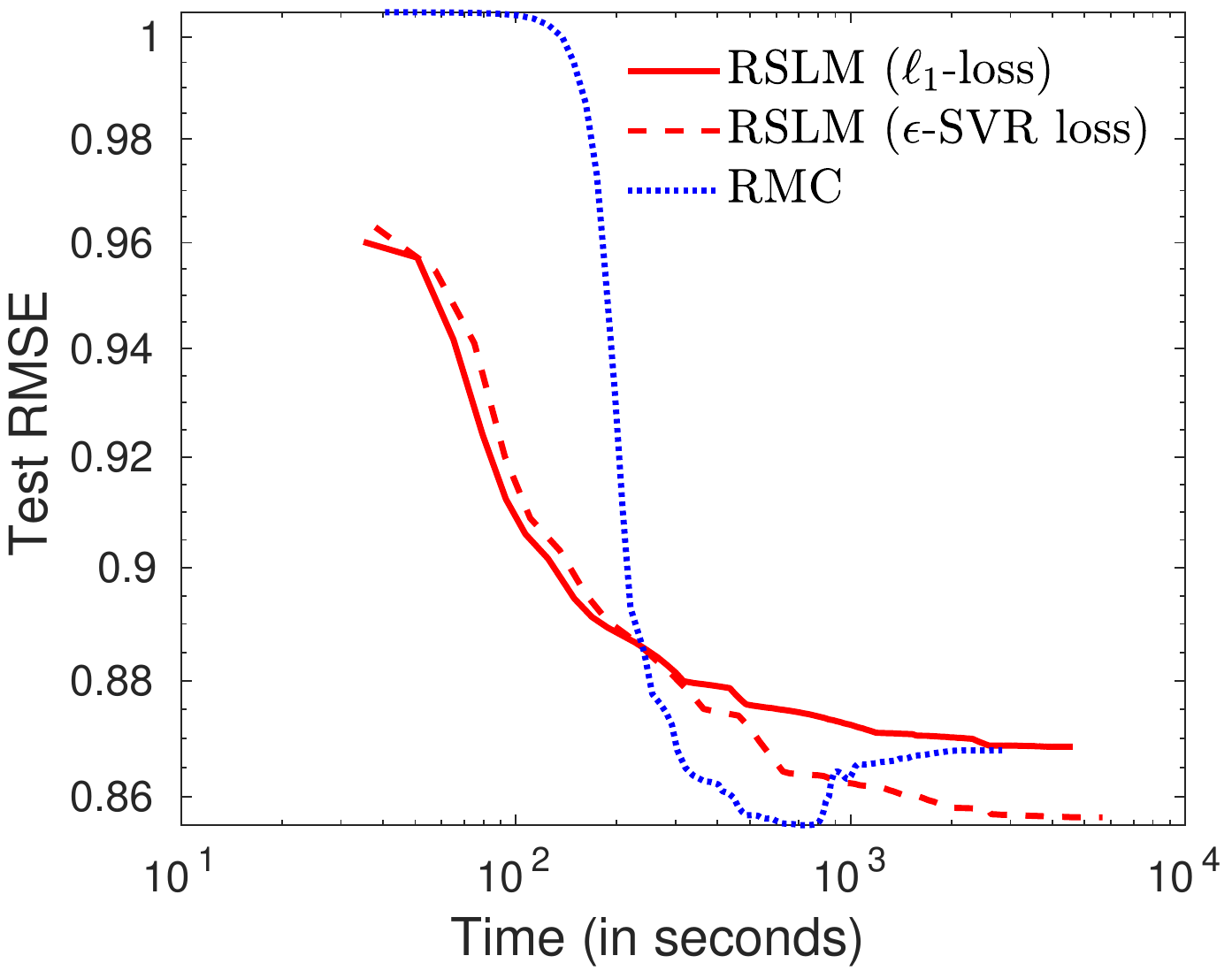}%
    \\
    {\small \ThreeLabels {(a) First-order MC methods} {(b) Second-order MC methods} {(c) Robust MC methods}}
%   \hfill(a) First-order algorithms\hfill\ \ \ \ \ \  \ \ \ \ \ \ \  \ \ \ \ \ \ \  \ (b) Second-order algorithms\hfill
    \caption{Evolution of test RMSE on the Netflix data set for first- and second-order matrix completion (MC) algorithms as well as robust MC algorithms. Our algorithms converge to the best generalization performance in all the experiments.}\label{fig:firstSecondOrderNetflix}
\end{figure*}

%\begin{table}\centering
%{
%\centering
%\caption{Data set statistics for the matrix completion application.}\label{appendix:table:datasetStatsMatrixCompletions}
%\begin{tabular}{llll}
%\toprule
%  & \multicolumn{1}{c}{$d$} & \multicolumn{1}{c}{$T$} & \multicolumn{1}{c}{$|\Omega|$}\\
%\midrule 
%Netflix & $17\, 770$ & $480\, 189$ & $100\, 198\, 805$\\
%ML10m & $10\, 677$ & $71\, 567$ & $10\, 000\, 054$ \\
%ML20m & $26\, 744$ & $138\, 493$ & $20\, 000\, 263$ \\
%\bottomrule
%\end{tabular}
%}
%\end{table}

\textbf{Synthetic data set results.} We choose $d=5\, 000$, $T=500\, 000$ and $r=5$ to create a synthetic data set (with $<1\%$ observed entries), following the procedure detailed in  \citep{boumal11a,boumal15a}. The number of observed entries for both training ($|\Omega|$) and testing was $15\, 149\, 850$. 
The generalization performance of different methods is shown in Figure~\ref{appendix:fig:multipleMatrixCompletionFigs}(a). For the same run, we also plot the variation of the relative duality gap across iterations for our algorithms in figure~\ref{appendix:fig:multipleMatrixCompletionFigs}(b). 
It can be observed that RSLM-tr approach the global optima for the nuclear norm regularized problem \ref{eqn:genericPrimal0} in very few iterations and obtain test RMSE $\approx2.46\times 10^{-7}$.  Our first order algorithm, RSLM-cg, also achieves low  test RMSE and low relative duality gap.  It should be noted that our methods are able to exploit the condition that $d\ll T$ (rectangular matrices),  similar to RTRMC.

\textbf{Real-world data set results:} 
%We tested the methods on three real world data sets: Netflix~\citep{recht13a}, MovieLens10m (ML10m) and MovieLens20m (ML20m)~\citep{movielens}. 
%Their statistics are given in Table~\ref{appendix:table:datasetStatsMatrixCompletions}.  
Figures~\ref{fig:firstSecondOrderNetflix} (a)\&(b) display the evolution of root mean squared error on the test set (test RMSE) against the training time on the Netflix data set for first- and second-order algorithms, respectively. 
RSLM-cg is among the most efficient first-order method and RSLM-tr is the best second-order method. 
Table~\ref{appendix:table:matrixCompletionResult} reports the test RMSE  obtained by all the algorithms on the three data sets.  
We observe that both our algorithms obtain the best generalization performance. %compared to other convex approaches (APGL and Active ALT). 

\begin{table}
\centering
\caption{Generalization performance of various algorithms on the matrix completion problem. The table reports mean test RMSE along with the standard deviation over five random train-test split. The proposed first-order (RSLM-cg) and second-order (RSLM-tr) algorithms achieve the lowest test RMSE.}\label{appendix:table:matrixCompletionResult}
\begin{tabular}{llll}
\toprule
  & \multicolumn{1}{c}{Netflix} & \multicolumn{1}{c}{MovieLens10m} & \multicolumn{1}{c}{MovieLens20m}\\
\midrule 
RSLM-tr & $\mathbf{0.8443\pm0.0001}$ & $\mathbf{0.8026\pm0.0005}$ & $\mathbf{0.7962\pm0.0003}$\\
RSLM-cg & $\mathbf{0.8449\pm0.0003}$ & $\mathbf{0.8026\pm0.0005}$ & $\mathbf{0.7963\pm0.0003}$ \\
R3MC & $0.8478\pm0.0001$ & $0.8070\pm0.0004$ & $0.7982\pm0.0003$ \\
RTRMC & $0.8489\pm0.0001$ & $0.8161\pm0.0004$ & $0.8044\pm0.0005$\\
APGL & $0.8587\pm0.0005$ & $0.8283\pm0.0009$ & $0.8160\pm0.0013$\\ 
Active ALT & $0.8463\pm0.0005$ & $0.8116\pm0.0012$ & $0.8033\pm0.0008$\\ 
MMBS & $0.8454\pm0.0002$ & $0.8226\pm0.0015$ & $0.8053\pm0.0008$\\ 
LMaFit & $0.8484\pm0.0001$ & $0.8082\pm0.0005$ & $0.7996\pm0.0003$\\
PRP & $0.8488\pm0.0007$ & $0.8068\pm0.0006$ & $0.7987\pm0.0008$ \\
\bottomrule
\end{tabular}
\end{table}

%\begin{figure}[t]\centering
%{
%\includegraphics[width=0.4\columnwidth]{figures/RobustMatrixCompletion/Netflix/rmse_time_robust_Netflix_cg_abs_svr.pdf}%
%}
%\caption{Evolution of test RMSE of robust matrix completion algorithms on the Netflix data set. The proposed approach, RSLM (with $\epsilon$-SVR loss), converges with better generalization performance than RMC. \label{fig:MatrixCompletionNetflix}}
%\end{figure}

\subsection{Robust matrix completion}
%The formulation and algorithm of our approach (RSLM) for robust matrix completion are discussed in Section~\ref{sec:rmcformulation} and Algorithm~\ref{alg:trust_region}, respectively. 
We develop RSLM with two different robust loss functions: $\ell_1$-loss and $\epsilon$-SVR loss. The non-smooth nature of $\ell_1$-loss and $\epsilon$-SVR loss makes them challenging to optimize in large-scale low-rank setting. 

\textbf{Baseline techniques:}  
For evaluation, we compare RSLM against RMC~\citep{Cambier16}, a large-scale state-of-the-art Riemannian optimization algorithm for robust matrix completion. RMC employs the pseudo-Huber loss function  (\ref{eqn:rmcOrg}). While solving (\ref{eqn:rmcOrg}), RMC decreases the value of $\delta$ at regular intervals. Hence, at later iterations, when the value of $\delta$ is low, the pseudo-Huber loss approximates the $\ell_1$-loss. 
%approximate $\ell_1$-loss function via successive smoothing of the pseudo-Huber loss function (\ref{eqn:rmcOrg}). Successive smoothing implies that problem (\ref{eqn:rmcOrg}) is solved by decreasing the value of $\delta$ at regular intervals. Hence, at later iterations, when $\delta$ is a very low value, the loss function is close to the $\ell_1$-loss. 
%allows this loss function to be similar to the $\ell_1$-loss at later iterations. 
%We compare the following robust matrix completion algorithms: RMC~\citep{Cambier16}: state-of-the-art first-order Riemannian optimization algorithm that employs the smooth pseudo-Huber loss function (which successively approximates absolute loss), Proposed-cg-ab: our first-order algorithm with $\ell_1$-loss , and Proposed-cg-svr: our first-order algorithm employing $\epsilon$-SVR loss. %It should be noted that both $\ell_1$-loss and $\epsilon$-SVR loss functions are non-smooth.  (Section~\ref{sec:mtflformulation} and Algorithm~\ref{alg:trust_region}) 
% We compare our algorithms (Proposed-cg-ab and Proposed-tr-ab) with state-of-the-art robust matrix completion solver RMC~\citep{Cambier16}. RMC is a first-order Riemannian optimization algorithm and minimizes the smooth pseudo-Huber loss function (which successively approximates absolute loss). In contrast, our methods employ the non-smooth absolute-value loss function.  
%All the three loss functions are known to be robust to noise. 

\textbf{Data sets and experimental setup:} We compare the performance of all three algorithms on the Netflix data set. We follow the same experimental setup as described for the case of matrix completion. Rank $r=10$ is fixed for all the algorithms. 
%Section~\ref{sec:matrixcompletionexperiment}.

\textbf{Results:}
Figure~\ref{fig:firstSecondOrderNetflix}(c) shows the results on the Netflix data set. We observe that RSLM scales effortlessly on the Netflix data set even with non-smooth loss functions and obtains the best generalization performance with the $\epsilon$-SVR loss. The test RMSE obtained at convergence are: $\mathbf{0.857}$ (RSLM with $\epsilon$-SVR loss), $0.869$ (RSLM with $\ell_1$-loss), and $0.868$ (RMC).

\subsection{Non-negative matrix completion}
%Section~\ref{sec:hmlformulation} discusses the details of our approach (RSLM) for non-negative matrix completion problem. 
%The formulation and algorithm of our approach (RSLM) for NNMC are discussed in Section~\ref{sec:nnmcformulation} and Algorithm~\ref{alg:trust_region}, respectively. 
 
\textbf{Baseline techniques:}  We include the following methods in our comparisons: BMC~\citep{fang17a} and {BMA}~\citep{kannan12a}. BMC is a recently proposed ADMM based algorithm with carefully designed update rules to ensure an efficient computational and space complexity. BMA is based on co-ordinate descent algorithm. 
%ActiveALT is a traditional matrix completion solver, which does not enforce the non-negativity constraint. It employs an active subspace selection algorithm.  

\textbf{Data sets and experimental setup:} We compare the performance of the above algorithms on three real world data sets (Table \ref{table:datasetnnmc}): MovieLens1m (ML1m), MovieLens10m (ML10m), and MovieLens20m (ML20m). The experimental setup is the same as described for the case of matrix completion. The performance of all the algorithms are evaluated at three ranks: $r=5,10,20$. 

\textbf{Results:} 
We observe in Table~\ref{table:nonnegative} that RSLM outperforms both BMC and BMA. The improvement obtained by RSLM over BMC is more pronounced with larger data sets. 
BMA is not able to run on MovieLens20m due to high memory and time complexity.

Figures~\ref{fig:objvstime}(a)-(e) compare the convergence behavior of RSLM and BMC for different values of parameters $C$ and $r$ on all three data sets. Both algorithms  aim to minimize the same primal objective (\ref{eqn:genericPrimal0}) for a given rank $r$. Though the proposed approach solves the proposed fixed-rank dual formulation (\ref{eqn:dual21}), we compute the corresponding primal objective value of every iterate for the plots. We observe that our algorithm RSLM is significantly faster than BMC in converging to a lower objective value. 

Figures~\ref{fig:rmsevstime}(a)-(e) plot the evolution of test RMSE against training time for algorithms RSLM and BMC on different data sets with different ranks (and the corresponding best $C$ parameter). For a given data set and rank, both RSLM and BMC chose the same value of $C$ via cross-validation. We observe that the RSLM outperforms BMC in converging to a lower test RMSE at a much faster rate. 

\begin{table*}\centering
{\small
\centering
\caption{\label{table:nonnegative}Mean test RMSE on non-negative matrix completion problems. The proposed algorithm RSLM obtains the best generalization performance. 
}
\begin{tabular}{lllll}
\toprule
\multicolumn{1}{c}{Data set} & \multicolumn{1}{c}{$r$}  & \multicolumn{1}{c}{RSLM} & \multicolumn{1}{c}{BMC} & \multicolumn{1}{c}{BMA} \\
\midrule 
\multirow{3}{*}{MovieLens1m} & $5$  & $\mathbf{0.8651\pm0.0014}$ & $0.8672\pm0.0014$ & $0.9476\pm0.0112$ \\
 											  & $10$ & $0.8574\pm0.0015$ & $\mathbf{0.8529\pm0.0014}$ & $0.9505\pm0.0025$ \\
 											  & $20$ & $\mathbf{0.8678\pm0.0016}$ & $0.8691\pm0.0012$ & $0.9520\pm0.0033$ \\
\midrule
\multirow{3}{*}{MovieLens10m} & $5$  & $\mathbf{0.8135\pm0.0004}$ & $0.8237\pm0.0004$ & $0.8989\pm0.0088$ \\
 											    & $10$ & $\mathbf{0.8031\pm0.0004}$ & $0.8038\pm0.0007$ & $0.8832\pm0.0055$ \\
 											    & $20$ & $\mathbf{0.8148\pm0.0008}$ & $0.8842\pm0.0061$ & $0.8904\pm0.0041$ \\
\midrule
\multirow{3}{*}{MovieLens20m} & $5$  & $\mathbf{0.8142\pm0.0005}$ & $0.8454\pm0.0021$ & \multicolumn{1}{c}{--} \\
 											    & $10$ & $\mathbf{0.8014\pm0.0003}$ & $0.8477\pm0.0004$ & \multicolumn{1}{c}{--} \\
 											    & $20$ & $\mathbf{0.8065\pm0.0015}$ & $0.9130\pm0.0013$ & \multicolumn{1}{c}{--} \\
\bottomrule
\end{tabular}
}
\end{table*}
\begin{figure*}[t]\centering
{
\includegraphics[width=0.2\columnwidth]{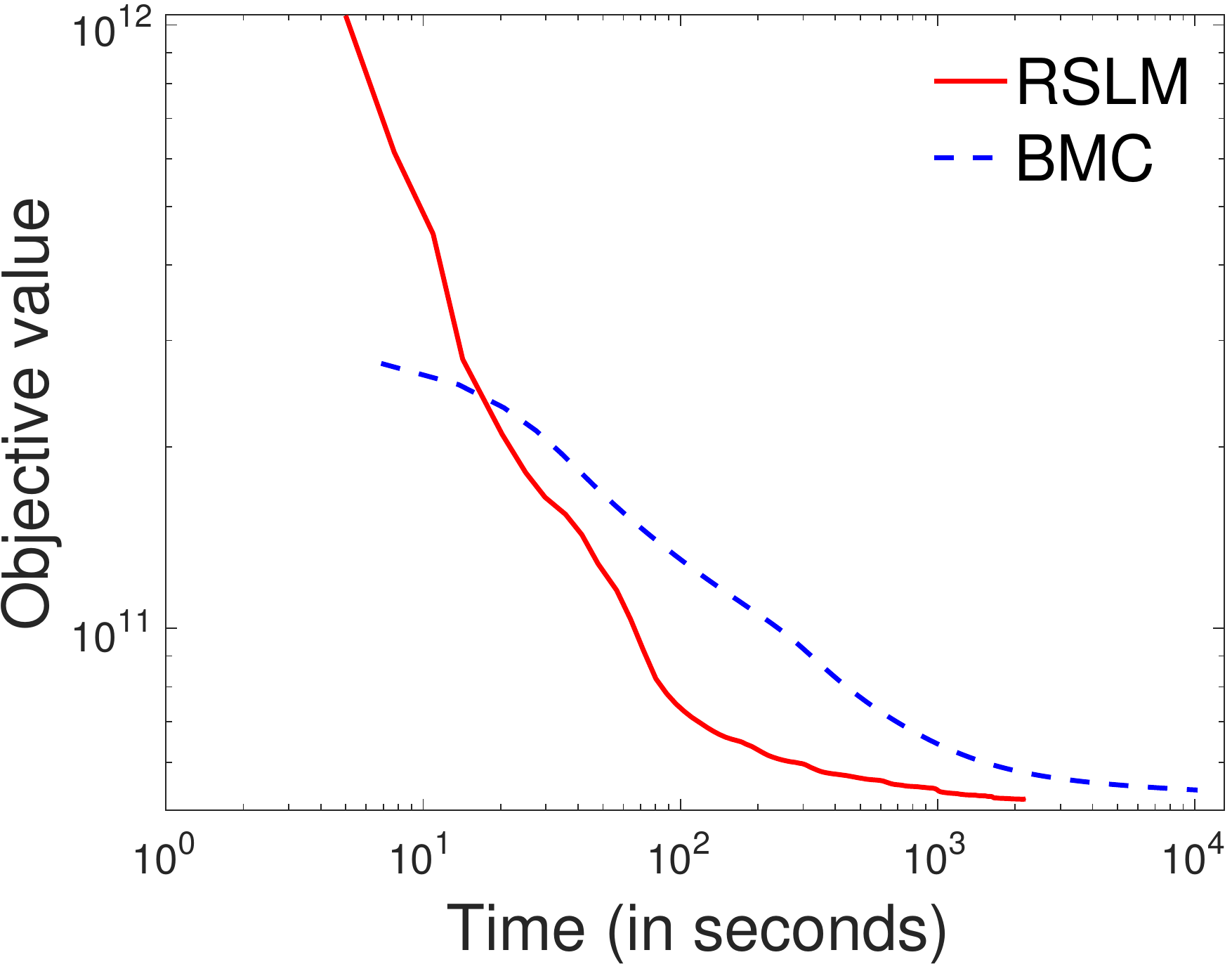}%
\hspace*{\fill}
\includegraphics[width=0.2\columnwidth]{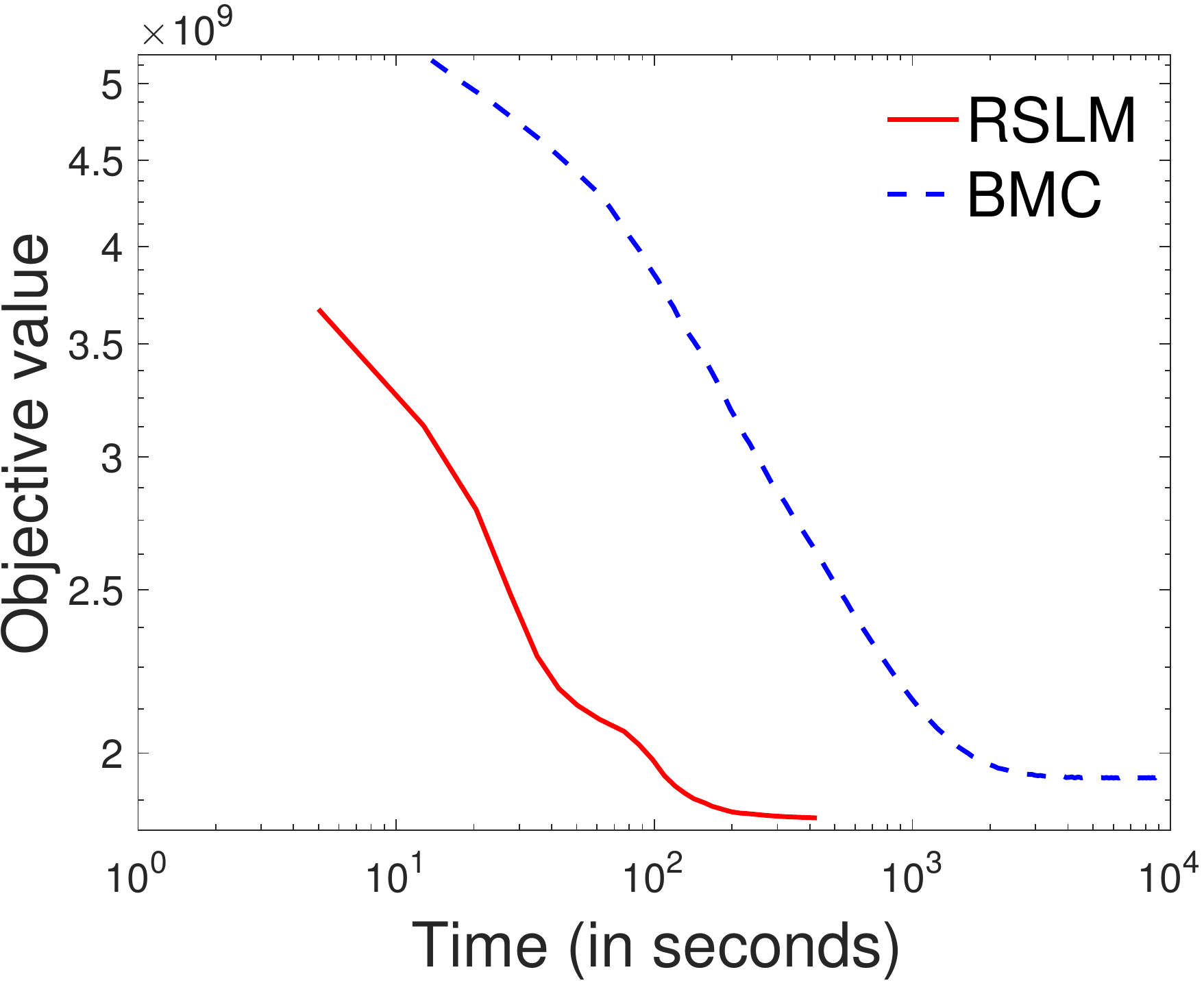}%
\hspace*{\fill}
\includegraphics[width=0.2\columnwidth]{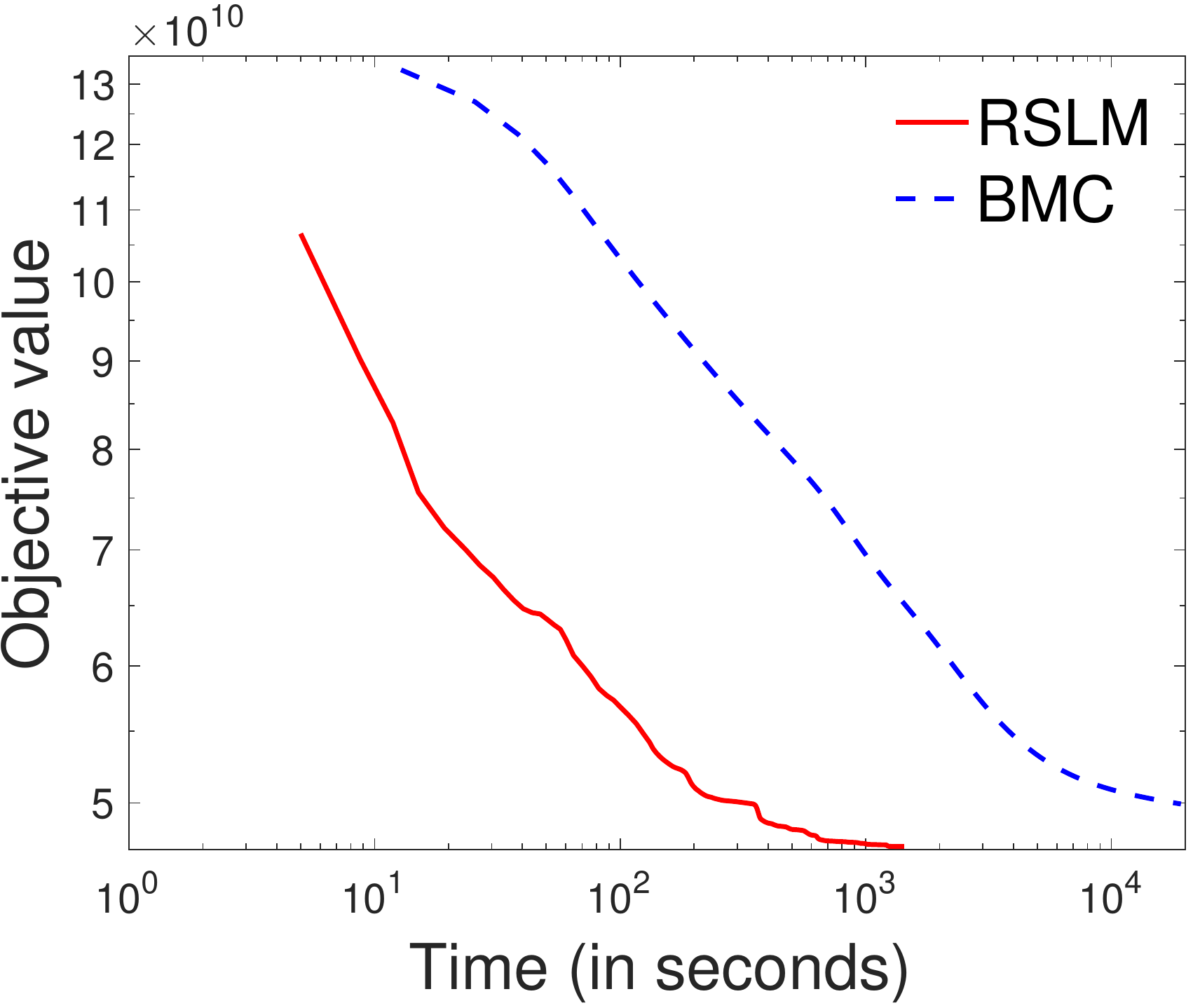}%
\hspace*{\fill}
\includegraphics[width=0.2\columnwidth]{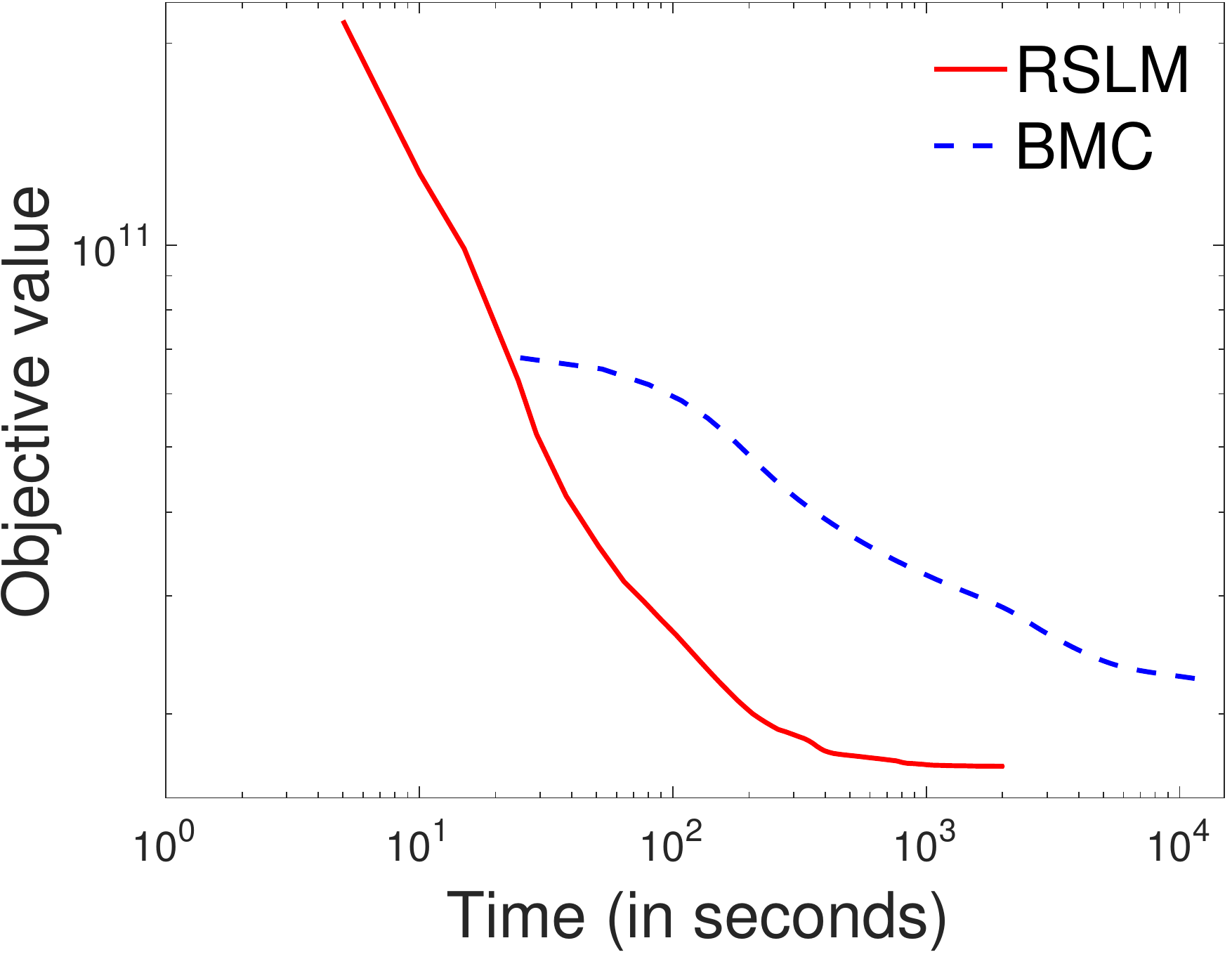}%
\hspace*{\fill}
\includegraphics[width=0.2\columnwidth]{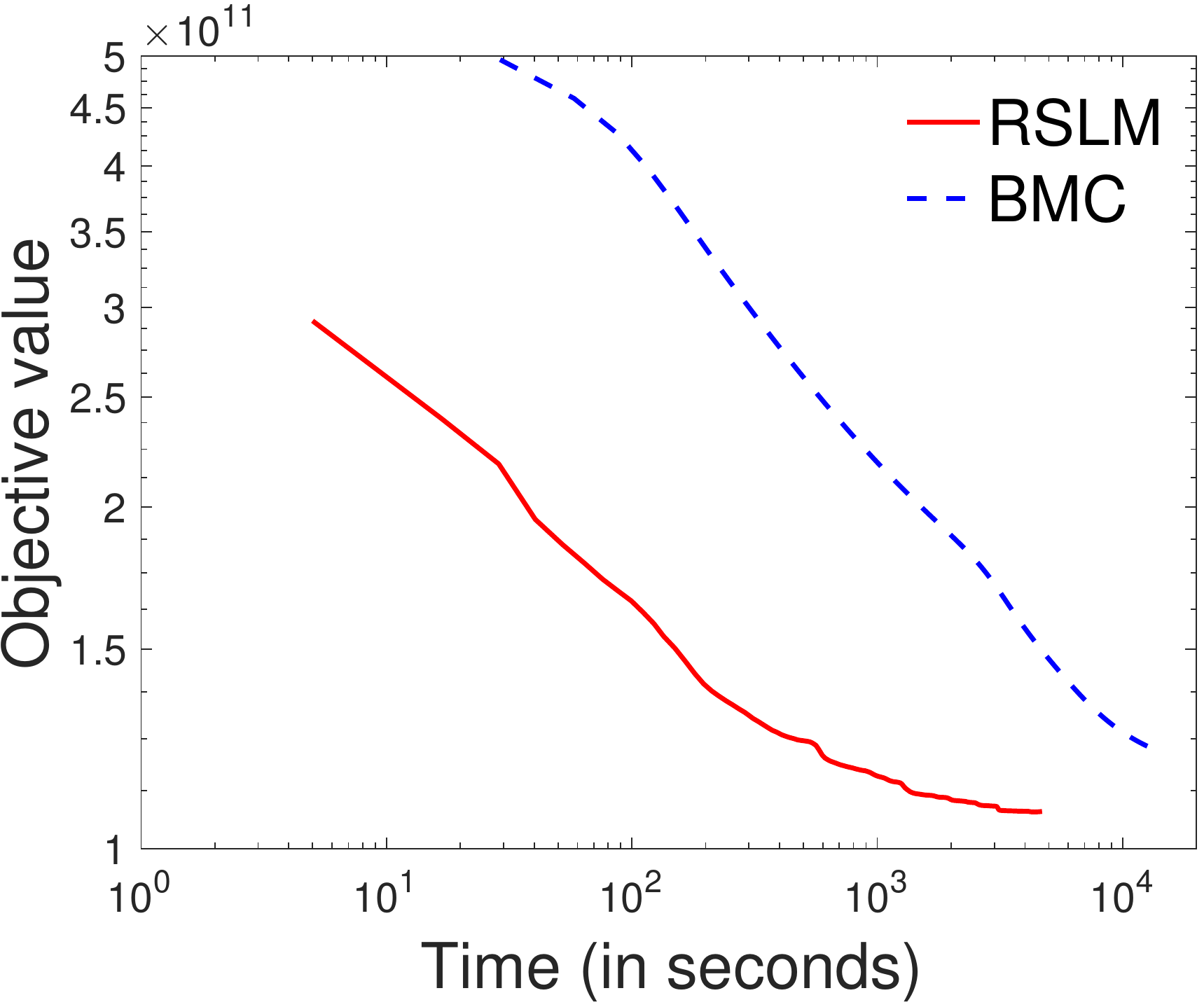}%
}
{\FiveLabels {(a)} {(b)} {(c)} {(d)} {(e)}}
\caption{Convergence behavior of RSLM and BMC for different values of regularization parameter and rank on non-negative matrix completion problems:  (a) ML10m, $(r,C):(5,1e4)$, (b) ML10m, $(r,C):(10,1e2)$, (c) ML10m, $(r,C):(10,1e4)$, (d) ML20m, $(r,C):(5,1e3)$, and (e) ML20m, $(r,C):(5,1e4)$. 
The proposed algorithm, RSLM, achieves better objective value in lesser training time than BMC. Both the axes are in $\log_{10}$ scale.}\label{fig:objvstime}%{\color{red} Both the axes are in $\log_{10}$ scale.}
\end{figure*}
\begin{figure*}[t]\centering
{
\includegraphics[width=0.2\columnwidth]{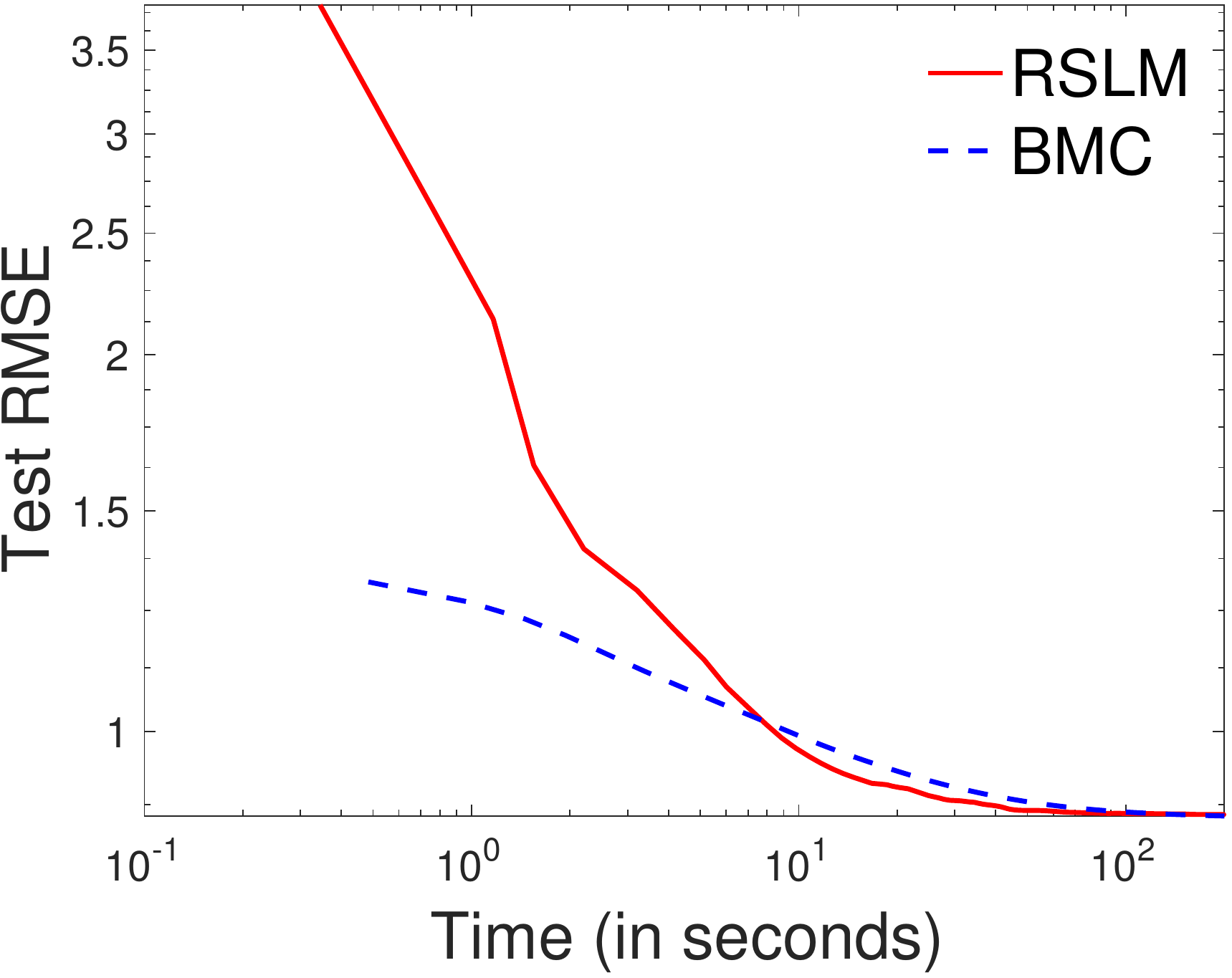}%
\hspace*{\fill}
\includegraphics[width=0.2\columnwidth]{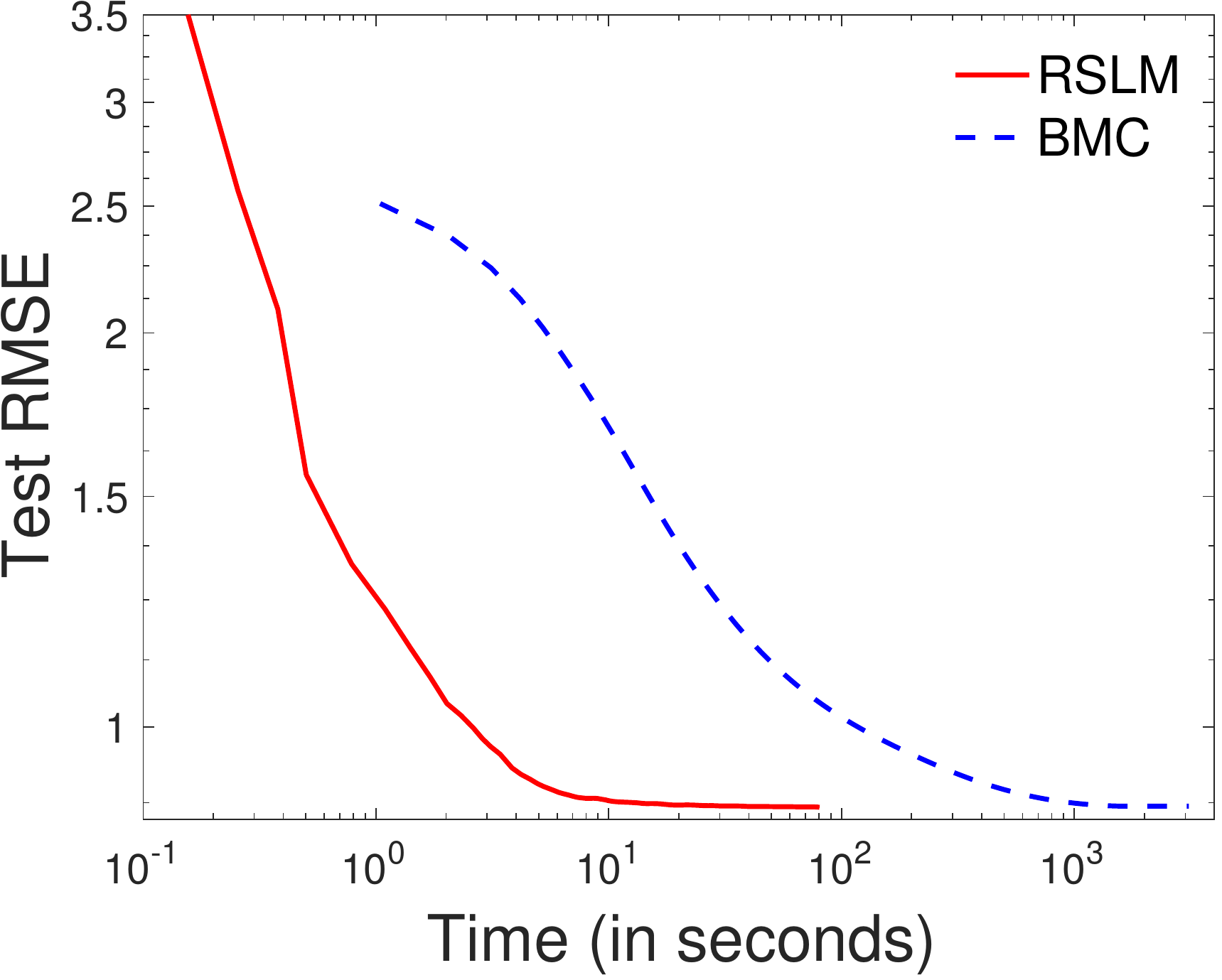}%
\hspace*{\fill}
\includegraphics[width=0.2\columnwidth]{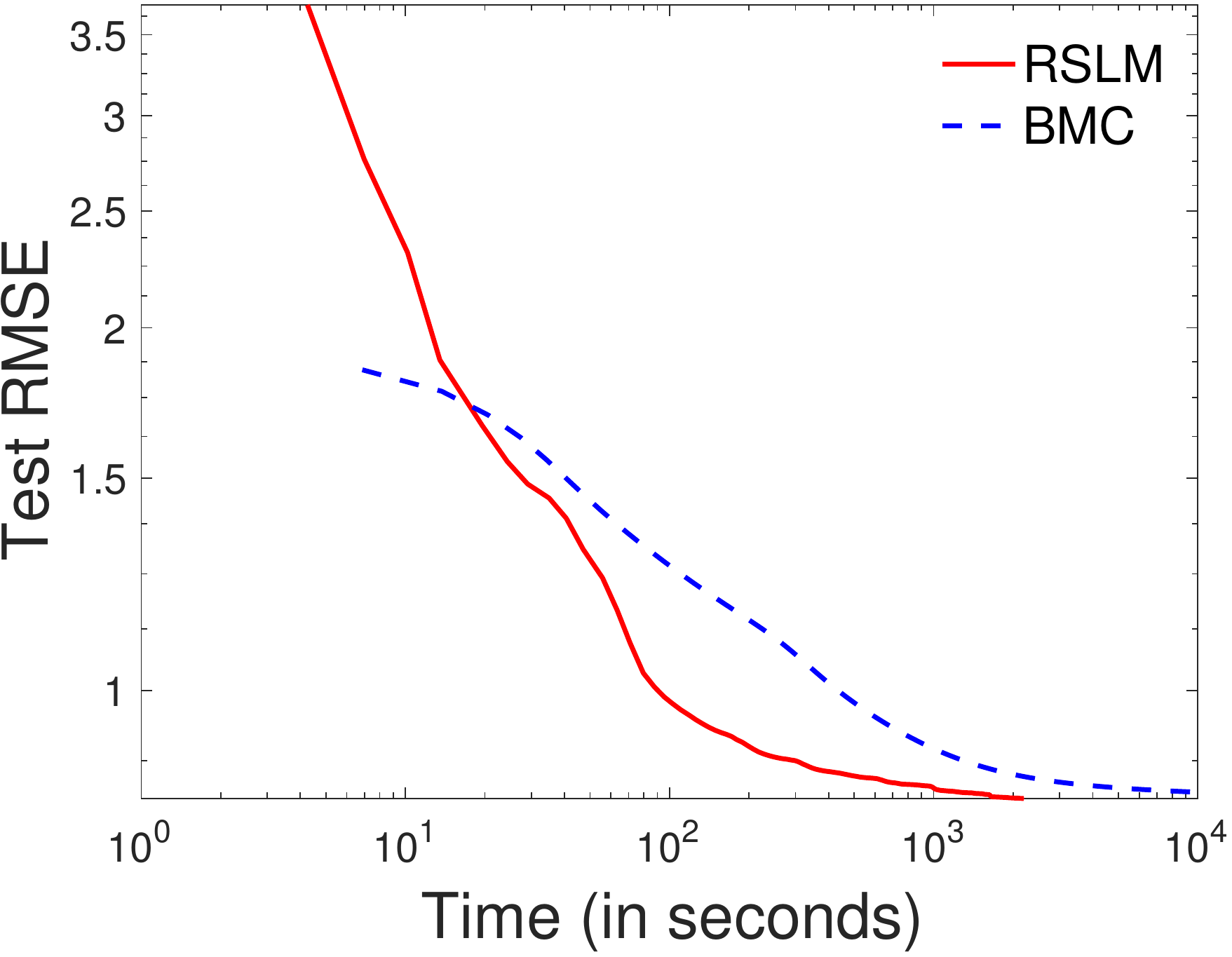}%
\hspace*{\fill}
\includegraphics[width=0.2\columnwidth]{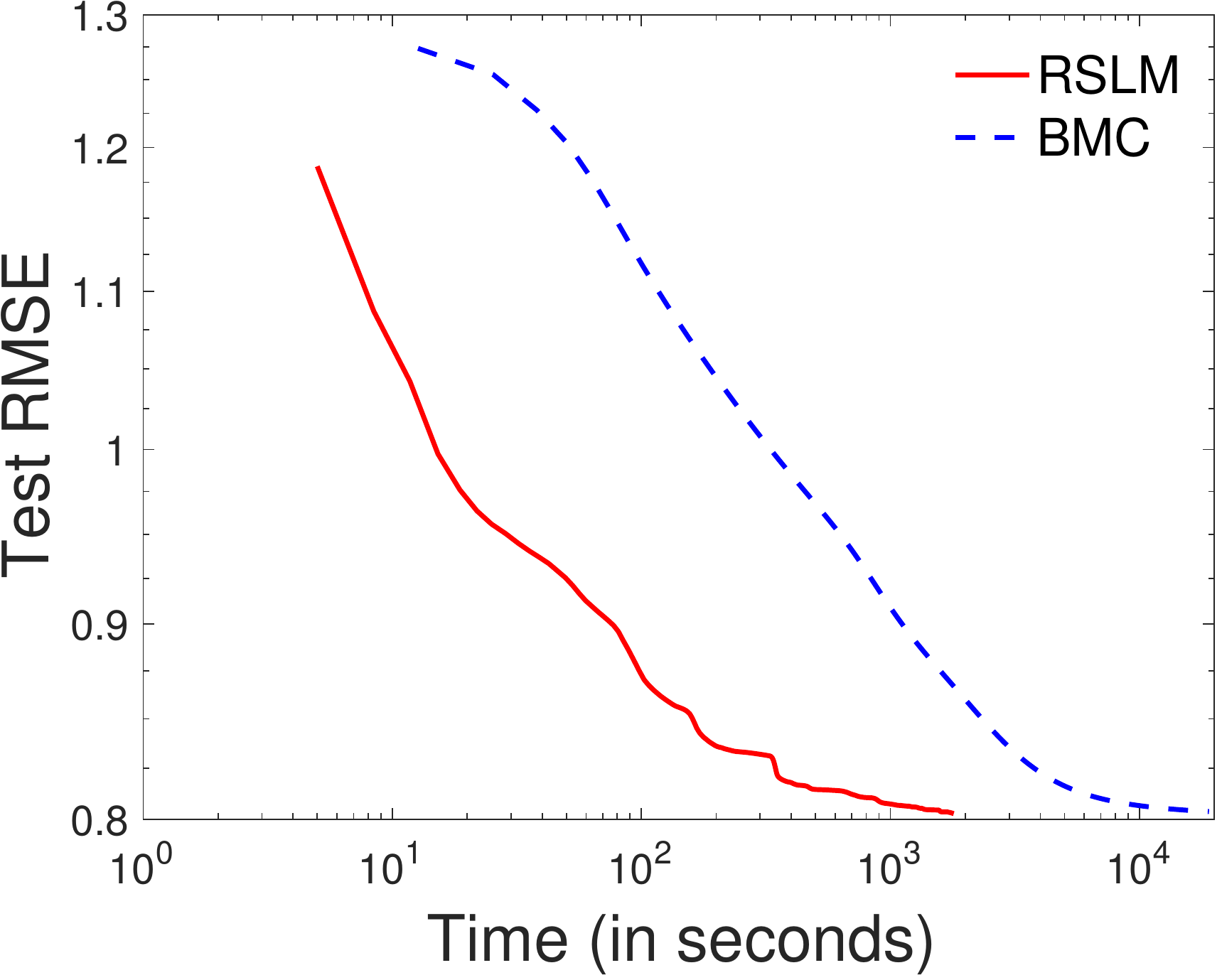}%
\hspace*{\fill}
\includegraphics[width=0.2\columnwidth]{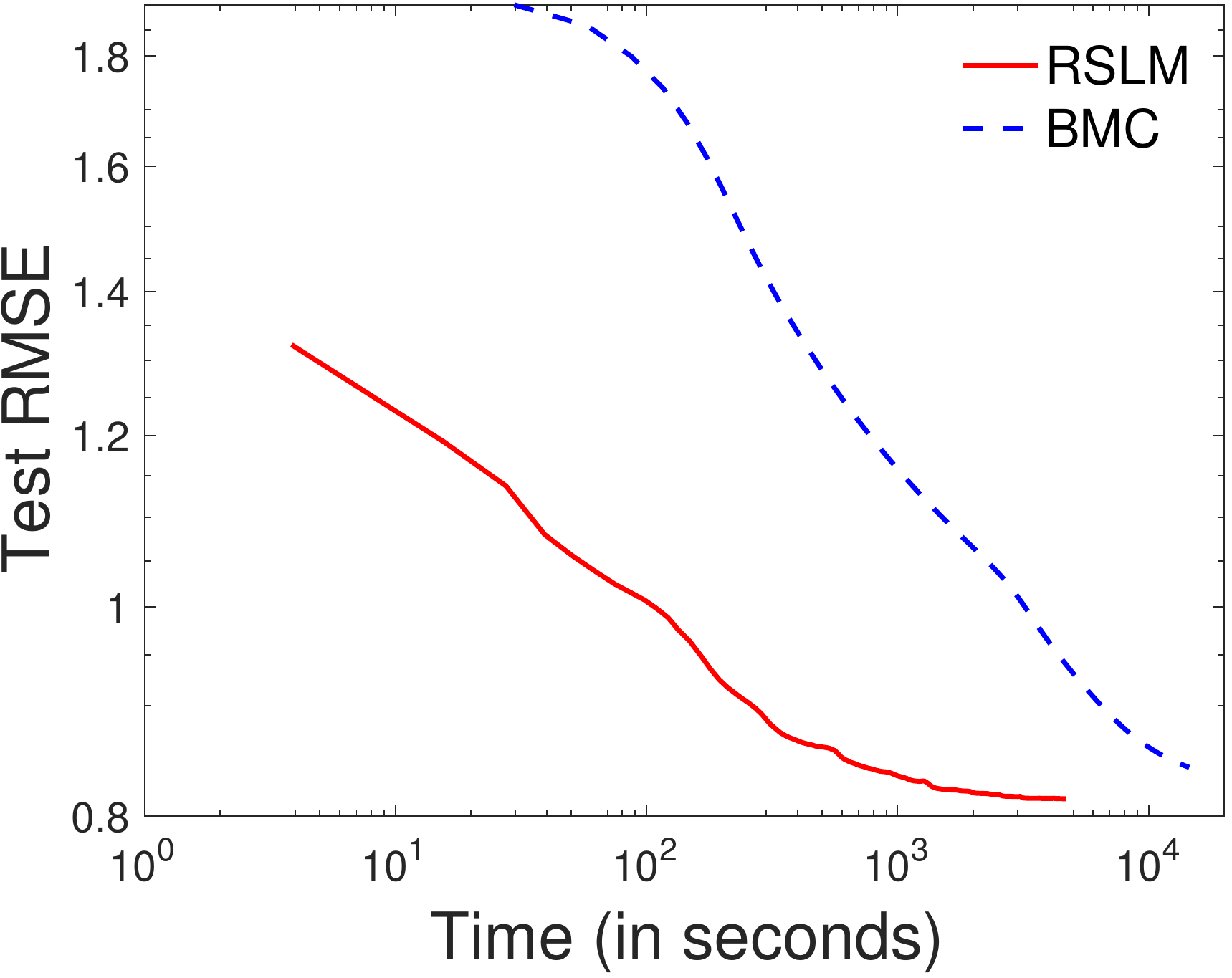}%
}
{\FiveLabels {(a)} {(b)} {(c)} {(d)} {(e)}}
\caption{Evolution of test RMSE of RSLM and BMC for different ranks (with tuned regularization parameter) on non-negative matrix completion problems: (a) ML1m, $(r,C):(10,1e3)$, (b) ML1m, $(r,C):(20,1e3)$, (c) ML10m, $(r,C):(5,1e4)$, (d) ML10m, $(r,C):(10,1e4)$, and (e) ML20m, $(r,C):(5,1e4)$. 
The proposed algorithm, RSLM, achieves better generalization performance than BMC and in lesser training time. Both the axes are in $\log_{10}$ scale. }\label{fig:rmsevstime}
\end{figure*}

\subsection{Hankel Matrix Learning}

\textbf{Baseline techniques:} We compare our algorithm RSLM with three methods (discussed in Section~\ref{sec:related}): GCG~\citep{yu14b}, SLRA~\citep{markovsky14a,Markovsky14}, and DADM~\citep{Fazel13}. Since GCG and DADM employ the nuclear norm regularizer, we tune the regularization parameter to vary the rank of their solution. 

\textbf{Data sets and experimental setup:}
Given a vector $y\in\R^{d+T-1}$, we obtain a $d\times T$ Hankel matrix as discussed in Section~\ref{sec:hmlformulation}. The true parameter $y$ is generated as the impulse response  (skipping the first sample) of a discrete-time random linear time-invariant system of order $r_0$ \citep{Markovsky11a,Fazel13}. The noisy estimate of these parameters ($\bar{y}$) are generated as $\bar{y}=y+\sigma\epsilon$, where $\sigma\epsilon$ is the measurement noise. We set $\sigma=0.05$ and generate $\epsilon$ from the standard Gaussian distribution $N(0,1)$.  It should be noted that $\bar{y}$ is the {\it train} data and $y$ is the {\it true} data. 

We generate three different data set $D_1,D_2,D_3$ of varying size and order. $D_1$ has $(r_0,d,T)=(5,100,100)$, where $r_0$ is the true order of the underlying system. 
The other two data sets, $D_2$ and $D_3$, has the configurations $(10,100,1000)$ and $(20,1000,10000)$, respectively. 
We evaluate the algorithms on all the three data sets and report their RMSE with respect to the true data~\citep{liu09a} at different ranks ($r=5,10,20$) in Table~\ref{table:hankel}.

\textbf{Results:} We observe from Table~\ref{table:hankel} that the proposed algorithm RSLM obtains the best result in all the three data sets and across different ranks. In addition, RSLM also usually obtains the lowest true RMSE for a data set at the rank $r$ equal to $r_0$ of the data set. This implies that RSLM is able to identify the minimal order of the systems corresponding to the data sets.

\begin{table}\centering
{\small
\centering
\caption{\label{table:hankel} RMSE on problems that involve learning a low-rank Hankel matrix. The proposed algorithm, RSLM, obtains the best generalization performance. 
}
\begin{tabular}{llllll}
\toprule
\multicolumn{1}{c}{Data set} & \multicolumn{1}{c}{$r$}  & \multicolumn{1}{c}{RSLM} & \multicolumn{1}{c}{SLRA} & \multicolumn{1}{c}{DADM} & \multicolumn{1}{c}{GCG}\\
\midrule 
\multirow{3}{*}{\shortstack{$D_1$:\\ $d=100,$\\$T=100$}} 				  & $5$  & $\mathbf{0.0156}$ & $0.0159$ & $0.0628$ & $0.1068$\\
 											  													  & $10$ & $\mathbf{0.0171}$ & $0.0190$ & $0.0291$ & $0.0846$\\
 											  													  & $20$ & $\mathbf{0.0177}$ & $0.0295$ & $0.0196$ & $0.0667$\\
\midrule
\multirow{3}{*}{\shortstack{$D_2$:\\$d=100,$\\$T=1000$}}				    & $5$  & $\mathbf{0.0059}$ & $0.0164$ & $0.0331$ & $0.0531$\\
 											  													        & $10$ & $\mathbf{0.0060}$ & $0.0077$ & $0.0250$ & $0.0386$\\
 											  													        & $20$ & $\mathbf{0.0071}$ & $0.0108$ & $0.0159$ & $0.0345$\\
\midrule
\multirow{3}{*}{\shortstack{$D_3$:\\$d=1000,$\\$T=10000$}}				    & $5$  & $\mathbf{0.0149}$ & $\mathbf{0.0149}$ & $0.0458$ & $0.0458$\\
 											  													  & $10$ & $\mathbf{0.0043}$ & $0.0049$ & $0.0314$ & $0.0340$\\
 											  													  & $20$ & $\mathbf{0.0039}$ & $0.0053$ & $0.0288$ & $0.0330$\\
\bottomrule
\end{tabular}
}
\end{table}

\begin{figure*}\centering
%\vspace*{0.75em}
{
\includegraphics[width=0.24\columnwidth]{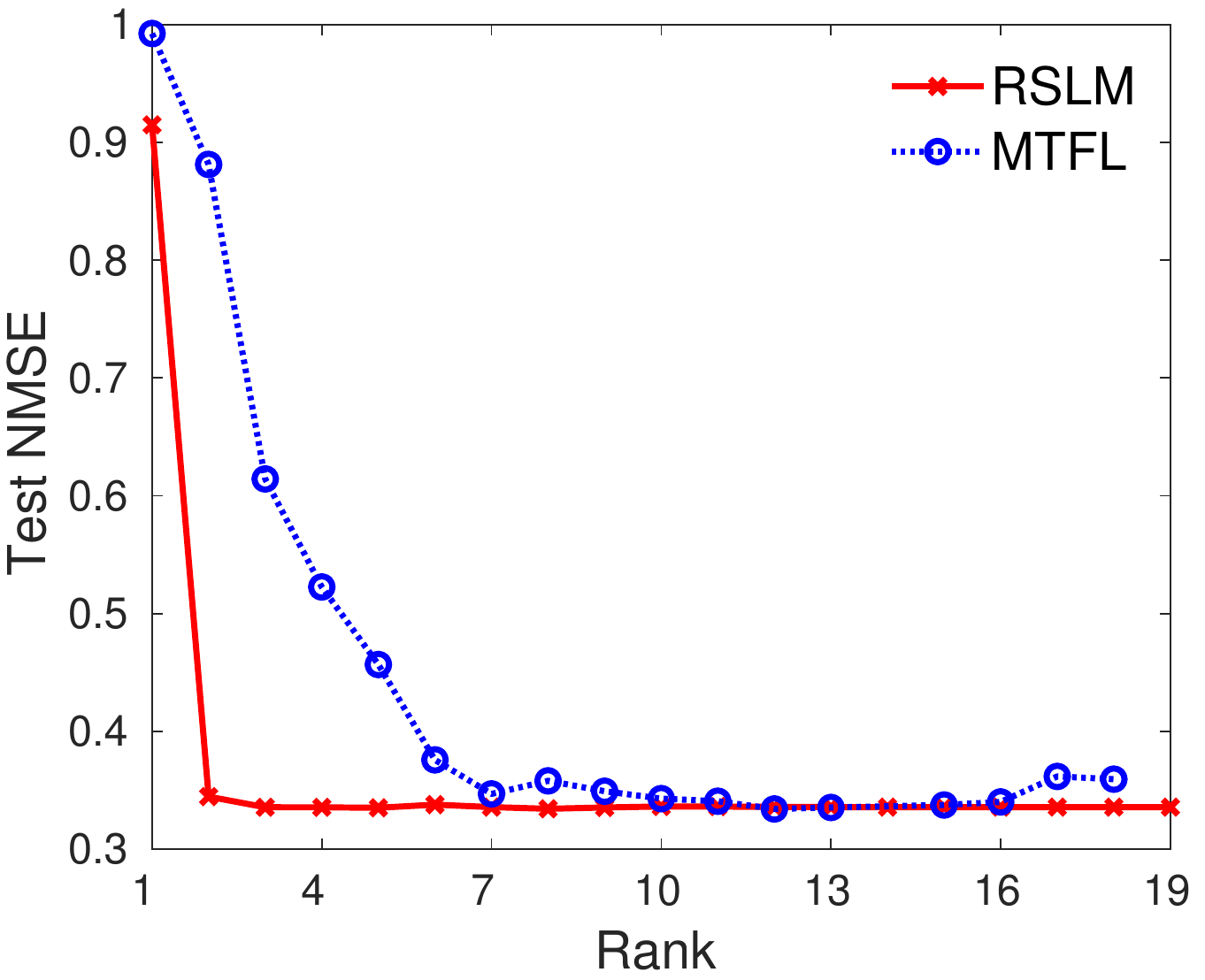}%
\hspace*{\fill}
\includegraphics[width=0.24\columnwidth]{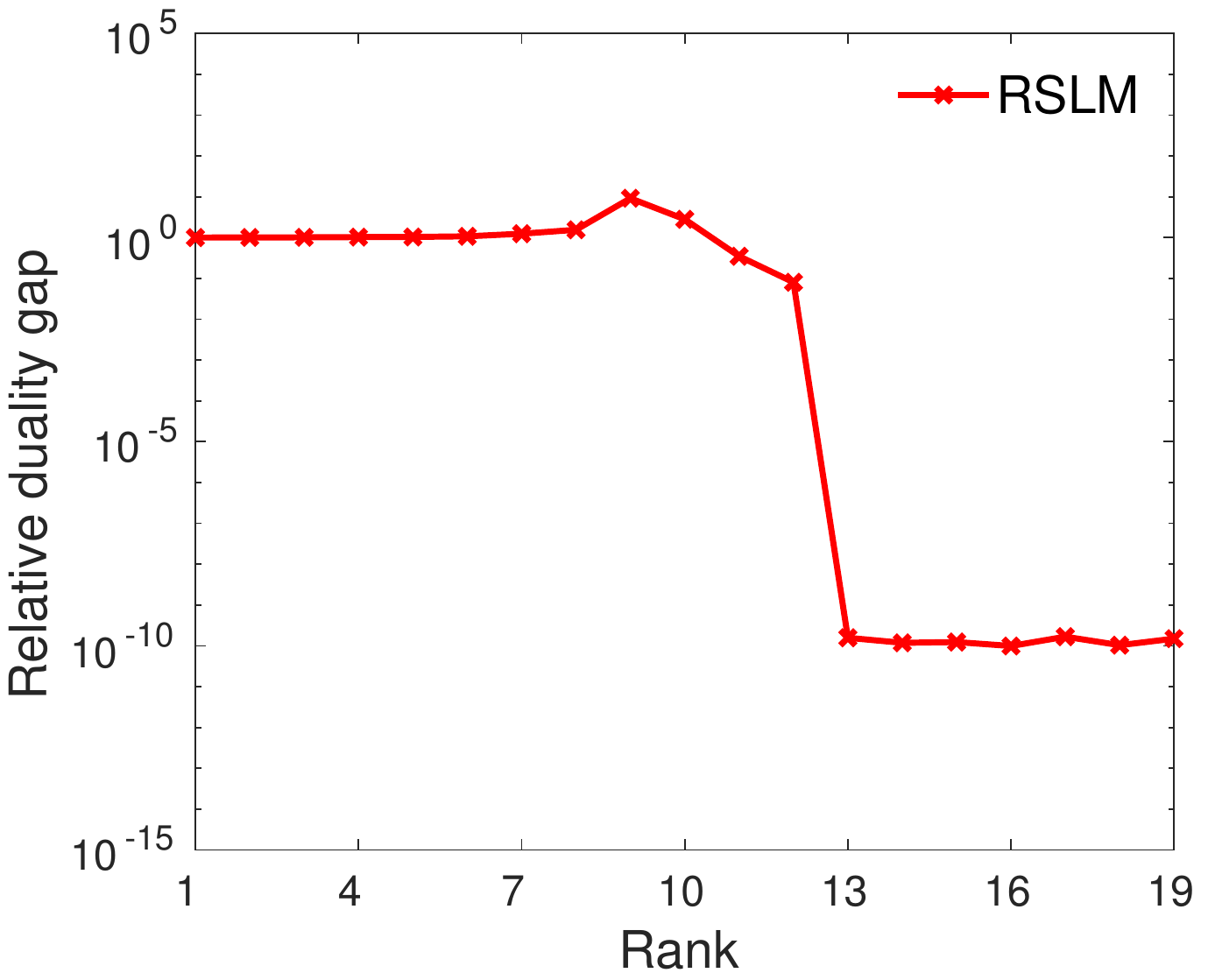}%
\hspace*{\fill}
\includegraphics[width=0.24\columnwidth]{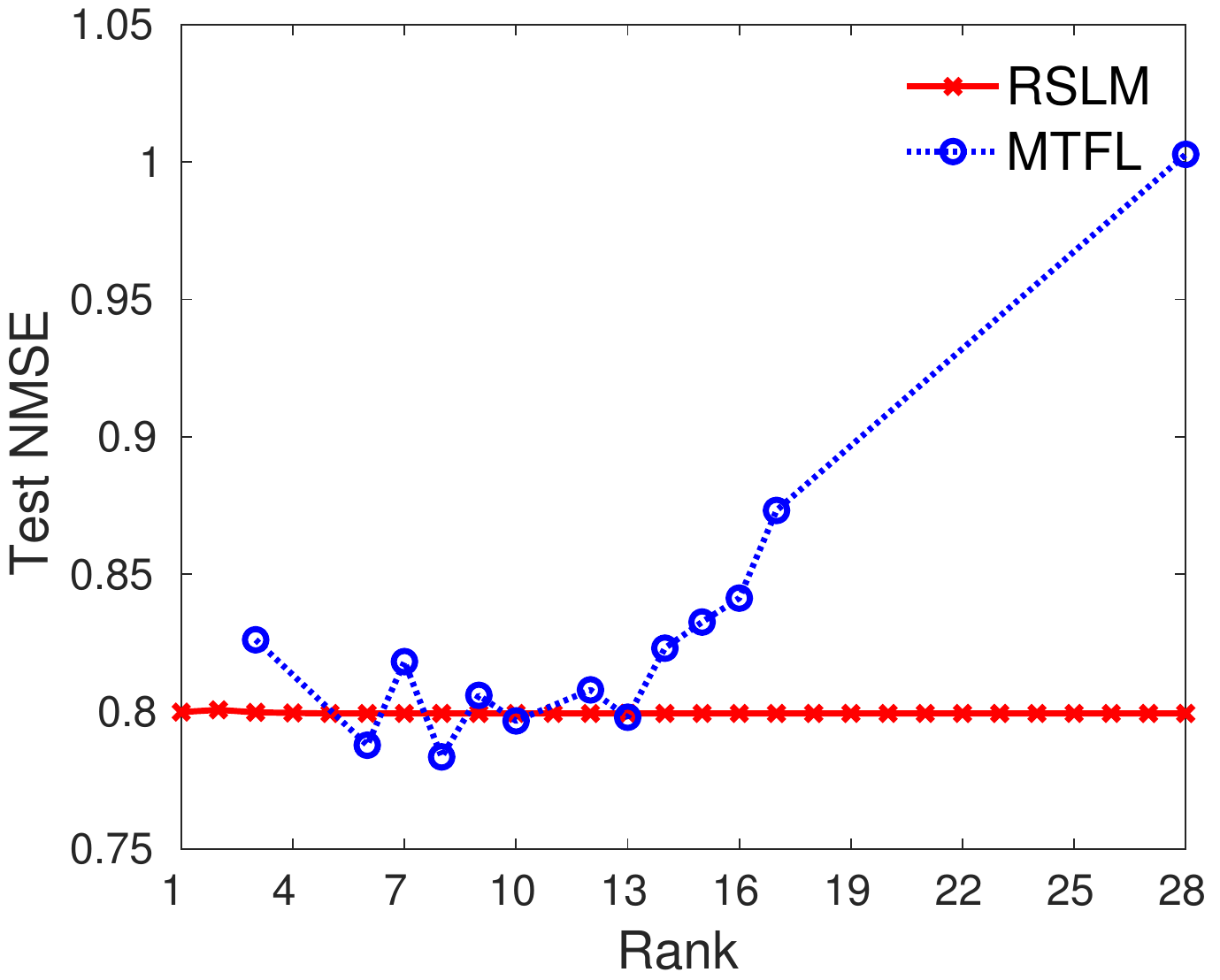}%
\hspace*{\fill}
\includegraphics[width=0.24\columnwidth]{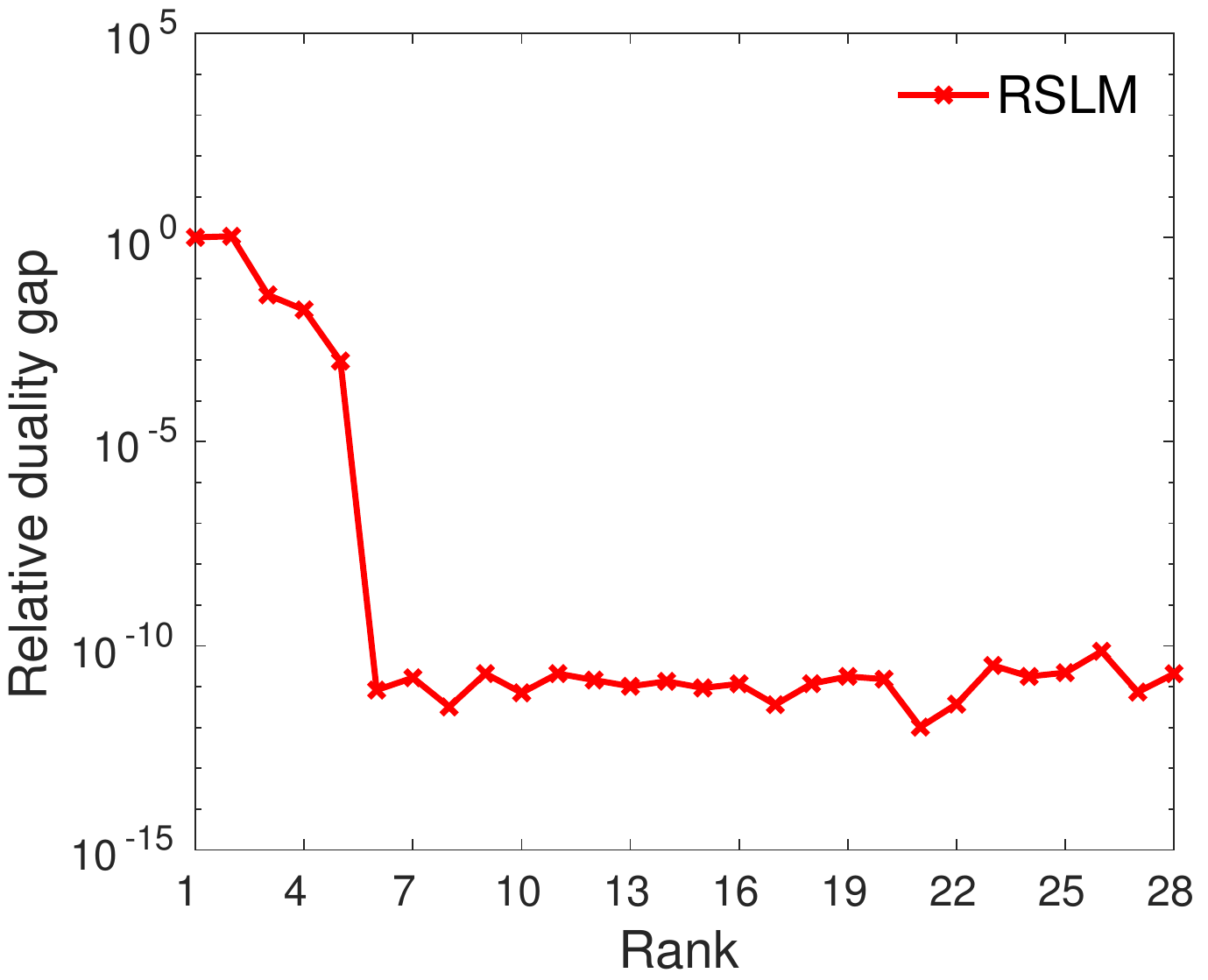}\\%
}
\FourLabels {(a)} {(b)} {(c)} {(d)}
\caption{%
(a) \& (c) Variation of normalized mean squared error (NMSE) as the rank of the optimal solution changes on Parkinsons and School data sets respectively. Our multi-task feature learning algorithm,  RSLM, obtains best generalization at much lower rank compared to state-of-the-art MTFL algorithm~\citep{Argyriou08}; (b) \& (d) The relative duality gap $(\Delta)$ corresponding to  the optimal solutions obtained by RSLM at different ranks. A small $\Delta$ implies that the solution obtained by RSLM is also the global optimal solution of the primal problem (\ref{eqn:genericPrimal0}).  
%Figure best viewed in color. 
}\label{appendix:fig:multiTask}
\end{figure*}
\subsection{Multi-task feature learning}
\textbf{Experimental setup:} We compare the generalization performance of our algorithm RSLM with the convex multi-task feature learning algorithm MTFL \citep{Argyriou08}. Optimal solution for MTFL at different ranks is obtained by tracing the solution path with respect to the regularization parameter, whose value is varied as $\{2^{-8},2^{-7},\ldots,2^{24}\}$. For RSLM, we fix the value of the regularization parameter $C$, and vary the rank $r$ to obtain different ranked solutions. The experiments are performed on two benchmark multi-task regression data sets: 
% a) Parkinsons: we need to predict the Parkinson's disease symptom score of $42$ patients~\citep{Frank10}; b) School: we need to predict performance of all students in $139$ schools. 
\begin{itemize}
\item Parkinsons: We need to predict the Parkinson's disease symptom score of 42 patients~\citep{Frank10}. Each patient is described using $19$ bio-medical features. The data set has a total of 5,875 readings from all the patients. 
\item School: The data consists of 15,362 students from 139 schools~\citep{Argyriou08}. The aim is to predict the performance of each student given their description and earlier record. Overall, each student data has 28 features. Predicting the performance of students belonging to one school is considered as one task.
\end{itemize}
Following~\citep{Argyriou08}, we report the normalized mean square error over the test set (test NMSE). 

\textbf{Results:} Figure~\ref{appendix:fig:multiTask}(a) present the results on the Parkinsons data set. We observe from the figure that our method achieves better generalization performance at low ranks compared to MTFL. 
Figure~\ref{appendix:fig:multiTask}(b) plots the variation of duality gap with rank for our algorithm. We observe that as the rank is varied from low to high, we converge to the globally optimal solution of (\ref{eqn:dual11}) and obtain the duality gap close to zero. 
Similar results are obtained on the School data set as observed in Figures~\ref{appendix:fig:multiTask} (c)\&(d).

\section{Conclusion}\label{sec:conclusion}
We have proposed a novel factorization for structured low-rank matrix learning problems, which stems from the application of duality theory and rank-constrained parameterization of positive semi-definite matrices. This allows to develop a conceptually simpler and unified optimization framework for various applications. State-of-the-art performance of our algorithms on several applications shows the effectiveness of our approach. 

\section*{Acknowledgement}
We thank L\'eopold Cambier, Ivan Markovsky, and Konstantin Usevich for useful discussions. We also thank Rodolphe Jenatton, Syama Rangapuram, and Matthias Hein for giving feedback on an earlier version of this work. We are grateful to Adams Wei Yu and Mingkui Tan for making their codes available. Most of this work was done when PJ and BM were at Amazon.com, India.

%By exploiting the duality theory, we propose a novel factorization for structured low-rank matrix learning. This results in a unified framework for several applications such as 
%
%The structured low-rank matrix learning problem is modeled as a novel rank-constrained saddle point optimization problem. The benefit of this modeling is that it decouples the low-rank and structural constraints onto separate factors. The optimization problem is shown to lie on the Riemannian spectrahedron manifold. The Riemannian structure enables to propose computationally efficient conjugate gradient and trust-region algorithms for structured low-rank matrix learning.
%%Our framework is able to employ even non-smooth loss functions such as the absolute-value loss, which are typically more challenging to optimize. 
%Our algorithms scale readily to the Netflix data set even with non-smooth loss functions such as the $\ell_1$-loss and $\epsilon$-SVR loss. We obtain state-of-the-art generalization performance on standard and robust matrix completion, low-rank Hankel matrix learning, and multi-task feature learning problems. 

%We show results with such loss function in the (robust) matrix completion application on Netflix data set. We also show the efficacy of our framework on low rank Hankel matrix learning problem (which has structural constraints) and in multi-task learning applications. 
% \bibliography{example_paper}
%\small

\appendix

\section{Optimization on Spectrahedron}\label{sec:optSpectrahedron}
We are interested in the optimization problem of the form
\begin{equation}\label{sup:eq:problem_formulation}
\begin{array}{lll}
\minop_{\Theta \in \P^d} & f(\Theta), \\
\end{array}
\end{equation}
where $\P^d$ is the set of $d\times d$ positive semi-definite matrices with unit trace and $f: \P^d \rightarrow \R $ is a smooth function. A specific interest is when we seek matrices of rank $r$. Using the parameterization $\Theta = \bU \bU^\top$, the problem (\ref{sup:eq:problem_formulation}) is formulated as 
\begin{equation}\label{sup:eq:fixed_rank_problem_formulation}
\begin{array}{lll}
\minop_{\bU \in \mathcal{S}_r^d} & f(\bU \bU^\top),
\end{array}
\end{equation}
where $\mathcal{S}_r^d\coloneqq \{\bU \in \mathbb{R}^{d \times r}: \| \bU\|_F = 1 \}$, which is called the spectrahedron manifold \citep{journee10a}. It should be emphasized the objective function in (\ref{sup:eq:fixed_rank_problem_formulation}) is \emph{invariant} to the post multiplication of $\bU$ with orthogonal matrices of size $r\times r$, i.e., $\bU \bU^\top =\bU \bQ (\bU \bQ)^\top$ for all $\bQ \in \OG{r}$, which is the set of orthogonal matrices of size $r\times r$ such that $\bQ \bQ^\top = \bQ^\top \bQ = \bI$. An implication of the this observation is that the minimizers of (\ref{sup:eq:fixed_rank_problem_formulation}) are no longer isolated in the matrix space, but are isolated in the quotient space, which is the set of equivalence classes $[\bU] \coloneqq \{ \bU \bQ: \bQ \bQ^\top = \bQ^\top \bQ = \bI \}$. Consequently, the search space is 
\begin{equation}\label{sup:eq:equivalence_manifold}
\begin{array}{lll}
\mathcal{M}  \coloneqq  \mathcal{S}_r^d/\OG{r}.
\end{array}
\end{equation}
In other words, the optimization problem (\ref{sup:eq:fixed_rank_problem_formulation}) has the structure of optimization on the \emph{quotient} manifold, i.e.,
\begin{equation}\label{sup:eq:manifold_formulation}
\begin{array}{lll}
\minop_{[\bU] \in \mathcal{M}} & f([\bU]),
\end{array}
\end{equation}
but numerically, by necessity, algorithms are implemented in the matrix space $\mathcal{S}_r^d$, which is also called the \emph{total space}.

Below, we briefly discuss the manifold ingredients and their matrix characterizations for (\ref{sup:eq:manifold_formulation}). Specific details of the spectrahedron manifold are discussed in \citep{journee10a}. A general introduction to manifold optimization and numerical algorithms on manifolds are discussed in \citep{absil08a}.

\subsection{Tangent vector representation as horizontal lifts}
Since the manifold $\mathcal{M}$, defined in (\ref{sup:eq:equivalence_manifold}), is an abstract space, the elements of its tangent space $T_{[\bU]} \mathcal{M}$ at $[\bU]$ also call for a matrix representation in {the tangent space $T_\bU \mathcal{S}_r^d$} that respects the equivalence {relation} $\bU \bU^\top =\bU \bQ (\bU \bQ)^\top$ for all $\bQ \in \OG{r}$. Equivalently, {the} matrix representation of $T_{[\bU]} \mathcal{M}$ should be restricted to the directions in the tangent space $T_{\bU}   \mathcal{S}_r^d$ on the total space $  \mathcal{S}_r^d$ at ${ \bU}$ that do not induce a displacement along the equivalence class $[\bU]$. In particular, we decompose  $T_\bU \mathcal{S}_r^d$ into complementary subspaces,  the \emph{vertical} $\mathcal{V}_{\bU}$ and \emph{horizontal} $\mathcal{H}_\bU$ subspaces, such that $ \mathcal{V}_{  \bU}  \oplus \mathcal{H}_{  \bU} = T_{  \bU}  \mathcal{S}_r^d$. 

\begin{table*}[t]
\begin{center} \scriptsize
\caption{Matrix characterization of notions on the quotient manifold $\mathcal{S}_r^d/\OG{r}$.}
\label{sup:tab:spaces} 
\begin{tabular}{ p{5.5cm} | p{5.5cm}} 
& \\
\toprule
& \\
Matrix representation of an element & $\bU $    \\ 
&  \\
Total space $\mathcal{S}_r^d$ & $\{\bU \in \R^{d\times r}: \|\bU\|_F = 1 \}$   \\ 
 &  \\
Group action &   $\bU \mapsto \bU \bQ$, where $\bQ \in \OG{r}$.  \\  
 & \\
Quotient space ${\mathcal M}$ & $\mathcal{S}_r^d /\OG{r}$  \\ 
& \\
Tangent vectors in the total space $\mathcal{S}_r^d $ at $\bU$ &$ \{ \bZ \in \R^{d \times r} :  \trace(\bZ^\top \bU) = 0 \}$  \\
&  \\
Metric between the tangent vector $\xi_\bU, \eta_\bU \in T_\bU \mathcal{S}_r^d$
& $\trace(\xi_\bU ^\top \eta_\bU)$ \\
& \\
Vertical tangent vectors at $\bU$ &  $\{\bU {\bf \Lambda}: {\bf \Lambda}\in \R^{r\times r},{\bf \Lambda}^\top = - {\bf\Lambda} \}$ \\
&  \\
Horizontal tangent vectors & $\{ \xi_\bU \in T_{\bU} \mathcal{S}_r^d : \xi_\bU^\top \bU = \bU^\top \xi_\bU\} $\\
&  \\
\bottomrule
\end{tabular}
\end{center} 
\end{table*}

The vertical space $\mathcal{V}_{ \bU}$ is the tangent space of the equivalence class $[\bU]$. On the other hand, the horizontal space $\mathcal{H}_{  \bU}$, which is {any complementary subspace} to $\mathcal{V}_{  \bU}$ in $T_\bU  \mathcal{S}_r^d$, provides a valid matrix representation of the abstract tangent space $T_{[\bU]} \mathcal{M}$. An abstract tangent vector $\xi_{[\bU]} \in T_{[\bU]} \mathcal{M}$ at $ [\bU]$ has a unique element in the horizontal space $ {\xi}_{ {\bU}}\in\mathcal{H}_{ {\bU}}$ that is called its \emph{horizontal lift}. {Our specific choice of the horizontal space is the subspace of $T_\bU  \mathcal{S}_r^d$ that is the \emph{orthogonal complement} of $\mathcal{V}_{  x}$ in the sense of a \emph{Riemannian metric}}.

The Riemannian metric at a point on the manifold is a inner product that is defined in the tangent space. An additional requirement is that the inner product needs to be \emph{invariant} along the equivalence classes \citep[Chapter~3]{absil08a}. One particular choice of the Riemannian metric on the total space $\mathcal{S}_r^d$ is
\begin{equation}\label{sup:eq:metric}
	 \langle {\xi}_{ {\bU}}, {\eta}_{ {\bU}} \rangle_ \bU:= \trace(\xi_\bU ^\top \eta_\bU),
\end{equation}
where $\xi_\bU, \eta_\bU \in T_\bU \mathcal{S}_r^d$. The choice of the metric (\ref{sup:eq:metric}) leads to a natural choice of the metric on the quotient manifold, i.e., 
\begin{equation}\label{sup:eq:metric_quotient}
	 \langle {\xi}_{ {[\bU]}}, {\eta}_{ [{\bU}]} \rangle_ {[\bU]}:= \trace(\xi_\bU ^\top \eta_\bU),
\end{equation}
where $\xi_{[\bU]}$ and $ \eta_{[\bU]}$ are abstract tangent vectors in $T_{[\bU]} \mathcal{M} $ and $\xi_\bU$ and $\eta_\bU$ are their horizontal lifts in the total space $\mathcal{S}_r^d$, respectively. Endowed with this Riemannian metric, the quotient manifold $\mathcal{M}$ is called a \emph{Riemannian} quotient manifold of $\mathcal{S}_r^d$.

Table \ref{sup:tab:spaces} summarizes the concrete matrix operations involved in computing horizontal vectors. 

Additionally, starting from an arbitrary matrix (an element in the ambient dimension $\R^{d\times r}$), two linear projections are needed: the first projection $\Psi_\bU$ is onto the tangent space $T_\bU \mathcal{S}_r^d$ of the total space, while the second projection $\Pi_\bU$ is onto the horizontal subspace $\mathcal{H}_\bU$. 

Given a matrix $\bZ \in \R^{d\times r}$, the projection operator $\Psi_\bU: \R^{d\times r} \rightarrow T_\bU \mathcal{S}_r^d : \bZ \mapsto \Psi_\bU(\bZ)$ on the tangent space is defined as 
\begin{equation}\label{sup:eq:tangent_space_projector}
\begin{array}{lll}
\Psi_\bU(\bZ) = \bZ - \trace(\bZ^\top  \bU) \bU. 
\end{array}
\end{equation}

Given a tangent vector $\xi_\bU \in T_\bU \mathcal{S}_r^d$, the projection operator $\Pi_\bU: T_\bU \mathcal{S}_r^d \rightarrow \mathcal{H}_\bU: \xi_\bU \mapsto \Pi_\bU(\xi_\bU)$ on the horizontal space is defined as 
\begin{equation}\label{sup:eq:horizontal_space_projector}
\begin{array}{lll}
\Pi_\bU(\xi_\bU) = \xi_\bU -  \bU {\bf \Lambda},
\end{array}
\end{equation}
where $\bf \Lambda$ is the solution to the \emph{Lyapunov} equation
\begin{equation*}
(\bU^\top \bU){\bf \Lambda} + {\bf \Lambda} (\bU^\top \bU) =  \bU^\top \xi_\bU- \xi_\bU^\top \bU.
\end{equation*}

\subsection{Retractions from Horizontal Space to Manifold}

An iterative optimization algorithm involves computing a search direction ({\it e.g.,} the gradient direction) and then “moving in that direction”. The default option on a Riemannian manifold is to move along geodesics, leading to the definition of the exponential map. Because the calculation of the exponential map can be computationally demanding, it is customary in the context of manifold optimization to relax the constraint of moving along geodesics. The exponential map is then relaxed to a \emph{retraction} operation, which is any map $R_\bU : \mathcal{H}_
\bU \rightarrow \mathcal{S}_r^d: \xi_\bU \mapsto R_\bU(\xi_\bU) $ that locally approximates the exponential map on the manifold \citep[Definition~4.1.1]{absil08a}. On the spectrahedron manifold, a natural retraction of choice is 
\begin{equation}\label{sup:eq:retraction}
R_\bU(\xi_\bU) := (\bU + \xi_\bU)/\|\bU + \xi_\bU \|_F,
\end{equation} 
where $\| \cdot \|_F$ is the Frobenius norm and $\xi_\bU$ is a search direction on the horizontal space $\mathcal{H}_\bU$.

An update on the spectrahedron manifold is, thus, based on the update formula
$\bU_+ = R_\bU (\xi_\bU) $.

\subsection{Riemannian Gradient and Hessian Computations}
The choice of the invariant metric (\ref{sup:eq:metric}) and the horizontal space turns the quotient manifold $\mathcal{M}$ into a \emph{Riemannian submersion} of $(\mathcal{S}_r^d, \langle \cdot, \cdot \rangle)$. As shown by \citep{absil08a}, this special construction allows for a convenient matrix characterization of the gradient and the Hessian of a function on the abstract manifold $\mathcal{M}$.

The matrix characterization of the Riemannian gradient is 
\begin{equation}\label{sup:eq:rgrad}
\begin{array}{lll}
\grad_\bU f = \Psi_\bU(\nabla_\bU f),
\end{array}
\end{equation}
where $\nabla_\bU f$ is the Euclidean gradient of the objective function $f$ and $\Psi_\bU$ is the tangent space projector (\ref{sup:eq:tangent_space_projector}).

An iterative algorithm that exploits second-order information usually requires the Hessian applied along a search direction. This is captured by the Riemannian Hessian operator $\hess$, whose matrix characterization, given a search direction $\xi_\bU \in \mathcal{H}_\bU$, is 
\begin{equation}\label{sup:eq:rhess}
\begin{array}{lll}
\hess_{\bU}[\xi_{\bU}] = \Pi_\bU\Big({\rm D}\nabla f[\xi_\bU] -  \trace((\nabla_\bU f)^\top \bU)\xi_\bU \\
 \quad \qquad  - \trace((\nabla_\bU f)^\top \xi_\bU + ({\rm D}\nabla f[\xi_\bU])^\top \bU)\bU\Big), \\ 
\end{array}
\end{equation}
where ${\rm D}\nabla f[\xi_\bU]$ is the directional derivative of the Euclidean gradient $\nabla_\bU f$ along $\xi_\bu$ and $\Pi_\bU$ is the horizontal space projector (\ref{sup:eq:horizontal_space_projector}).

Finally, the formulas in (\ref{sup:eq:rgrad}) and (\ref{sup:eq:rhess}) that the Riemannian gradient and Hessian operations require only the expressions of the standard (Euclidean) gradient of the objective function $f$ and the directional derivative of this gradient (along a given search direction) to be supplied.

%\bibliographystyle{icml2017}  
%\bibliography{structured_lowrank}

%\bibliographystyle{plainnat}
\bibliography{myreferences}
\end{document}